\documentclass[11pt,letterpaper]{article}
\usepackage{mdframed}
\usepackage{paralist}
\usepackage{caption}
\DeclareCaptionType{Algorithm}
\usepackage{newfloat}
\def\showauthornotes{1}
\def\showtableofcontents{0}
\def\showkeys{0}
\def\showdraftbox{0}
\def\showcolorlinks{1}
\def\usemicrotype{1}
\def\showfixme{0}

\def\writemode{0}

\usepackage{algorithmic,algorithm}
\usepackage{dsfont}

\usepackage{etex}

\usepackage[l2tabu, orthodox]{nag}

\usepackage{xspace,enumerate}

\usepackage[dvipsnames]{xcolor}

\usepackage[T1]{fontenc}
\usepackage[full]{textcomp}

\usepackage[american]{babel}

\usepackage{mathtools}

\usepackage{amsthm}

\newtheorem{theorem}{Theorem}[section]
\newtheorem*{theorem*}{Theorem}

\newtheorem{proposition}[theorem]{Proposition}
\newtheorem*{proposition*}{Proposition}
\newtheorem{lemma}[theorem]{Lemma}
\newtheorem*{lemma*}{Lemma}
\newtheorem{corollary}[theorem]{Corollary}
\newtheorem*{conjecture*}{Conjecture}
\newtheorem{fact}[theorem]{Fact}
\newtheorem*{fact*}{Fact}

\newtheorem*{hypothesis*}{Hypothesis}

\theoremstyle{definition}
\newtheorem{definition}[theorem]{Definition}

\newtheorem{openquestion}[theorem]{Open Question}

\theoremstyle{remark}
\newtheorem{claim}[theorem]{Claim}
\newtheorem*{claim*}{Claim}
\newtheorem{remark}[theorem]{Remark}
\newtheorem*{remark*}{Remark}

\newtheorem*{observation*}{Observation}

\ifnum\writemode=1
\usepackage[
letterpaper,
top=1.2in,
bottom=1.2in,
left=1in,
right=1in]{geometry}

\fi

\ifnum\writemode=0
\usepackage[
letterpaper,
top=0.7in,
bottom=0.9in,
left=1in,
right=1in]{geometry}
\fi

\usepackage{newpxtext} 
\usepackage{textcomp} 
\usepackage[varg,bigdelims]{newpxmath}
\usepackage[scr=rsfso]{mathalfa}
\usepackage{bm} 
\let\mathbb\varmathbb

\ifnum\showkeys=1
\usepackage[color]{showkeys}
\fi

\ifnum\showcolorlinks=1
\usepackage[
pagebackref,
colorlinks=true,
urlcolor=blue,
linkcolor=blue,
citecolor=OliveGreen,
]{hyperref}
\fi

\ifnum\showcolorlinks=0
\usepackage[
pagebackref,
colorlinks=false,
pdfborder={0 0 0}
]{hyperref}
\fi

\usepackage[capitalise,nameinlink]{cleveref}
\crefname{lemma}{Lemma}{Lemmas}
\crefname{definition}{Definition}{Definitions}

\newcommand{\Sref}[1]{\hyperref[#1]{\S\ref*{#1}}}

\usepackage{nicefrac}

\ifnum\usemicrotype=1
\usepackage{microtype}
\fi

\ifnum\showauthornotes=1
\newcommand{\Authornote}[2]{{\sffamily\small\color{red}{[#1: #2]}}}
\newcommand{\Authornotecolored}[3]{{\sffamily\small\color{#1}{[#2: #3]}}}
\newcommand{\Authorcomment}[2]{{\sffamily\small\color{gray}{[#1: #2]}}}
\newcommand{\Authorstartcomment}[1]{\sffamily\small\color{gray}[#1: }

\newcommand{\Authorfnote}[2]{\footnote{\color{red}{#1: #2}}}
\newcommand{\Authorfixme}[1]{\Authornote{#1}{\textbf{??}}}
\newcommand{\Authormarginmark}[1]{\marginpar{\textcolor{red}{\fbox{\Large #1:!}}}}
\else
\newcommand{\Authornote}[2]{}
\newcommand{\Authornotecolored}[3]{}
\newcommand{\Authorcomment}[2]{}
\newcommand{\Authorstartcomment}[1]{}

\newcommand{\Authorfnote}[2]{}
\newcommand{\Authorfixme}[1]{}
\newcommand{\Authormarginmark}[1]{}
\fi

\definecolor{forestgreen(traditional)}{rgb}{0.0, 0.27, 0.13}

\ifnum\showfixme=0

\fi

\usepackage{boxedminipage}

\newcommand{\Paren}[1]{\left(#1\right)}

\newcommand{\iprod}[1]{\langle#1\rangle}

\newcommand{\Esymb}{\mathbb{E}}

\DeclareMathOperator*{\E}{\Esymb}
\DeclareMathOperator*{\Var}{\Vsymb}

\newcommand{\textparen}[1]{\text{(#1)}}

\ifx\because\undefined
\newcommand{\because}[1]{\textparen{because #1}}
\else
\renewcommand{\because}[1]{\textparen{because #1}}
\fi

\newcommand\bdot\bullet

\ifx\mathds\undefined 

\else

\fi

\DeclareMathOperator{\poly}{poly}

\newcommand{\R}{\mathbb R}

\renewcommand{\leq}{\leqslant}
\renewcommand{\le}{\leqslant}
\renewcommand{\geq}{\geqslant}

\ifnum\showdraftbox=1
\newcommand{\draftbox}{\begin{center}
  \fbox{%
    \begin{minipage}{2in}%
      \begin{center}%
          \Large\textsc{Working Draft}\\%
        Please do not distribute%
      \end{center}%
    \end{minipage}%
  }%
\end{center}
\vspace{0.2cm}}
\else
\newcommand{\draftbox}{}
\fi

\let\epsilon=\varepsilon

\numberwithin{equation}{section}

\newcommand\MYcurrentlabel{xxx}

\newcommand{\MYstore}[2]{%
  \global\expandafter \def \csname MYMEMORY #1 \endcsname{#2}%
}

\newcommand{\MYload}[1]{%
  \csname MYMEMORY #1 \endcsname%
}

\newcommand{\MYnewlabel}[1]{%
  \renewcommand\MYcurrentlabel{#1}%
  \MYoldlabel{#1}%
}

\newcommand{\MYdummylabel}[1]{}

\newcommand{\torestate}[1]{%
  \let\MYoldlabel\label%
  \let\label\MYnewlabel%
  #1%
  \MYstore{\MYcurrentlabel}{#1}%
  \let\label\MYoldlabel%
}

\newcommand{\restatetheorem}[1]{%
  \let\MYoldlabel\label
  \let\label\MYdummylabel
  \begin{theorem*}[Restatement of \prettyref{#1}]
    \MYload{#1}
  \end{theorem*}
  \let\label\MYoldlabel
}

\newcommand{\restatelemma}[1]{%
  \let\MYoldlabel\label
  \let\label\MYdummylabel
  \begin{lemma*}[Restatement of \prettyref{#1}]
    \MYload{#1}
  \end{lemma*}
  \let\label\MYoldlabel
}

\newcommand{\restateprop}[1]{%
  \let\MYoldlabel\label
  \let\label\MYdummylabel
  \begin{proposition*}[Restatement of \prettyref{#1}]
    \MYload{#1}
  \end{proposition*}
  \let\label\MYoldlabel
}

\newcommand{\restatefact}[1]{%
  \let\MYoldlabel\label
  \let\label\MYdummylabel
  \begin{fact*}[Restatement of \prettyref{#1}]
    \MYload{#1}
  \end{fact*}
  \let\label\MYoldlabel
}

\newcommand{\restate}[1]{%
  \let\MYoldlabel\label
  \let\label\MYdummylabel
  \MYload{#1}
  \let\label\MYoldlabel
}

\newcommand{\addreferencesection}{
  \phantomsection
  \addcontentsline{toc}{section}{References}
}

\let\origparagraph\paragraph
\renewcommand{\paragraph}[1]{\origparagraph{#1.}}

\allowdisplaybreaks

\usepackage{paralist}

\usepackage{comment}

\usepackage{braket}

\usepackage{relsize}
\usepackage[font=footnotesize]{caption}
\usepackage{appendix}

\newcommand{\eqdef}{\stackrel{\mathrm{def}}{=}}
\newcommand{\bE}{\mathbb{E}}
\newcommand{\bR}{\mathbb{R}}
\newcommand{\oo}{\mathcal{O}}

\DeclareUrlCommand\email{}

\DeclareMathOperator*{\pE}{\tilde{\mathbb E}}

\newcommand{\tr}{\textup{tr}}

\newcommand{\goodset}{I}
\newcommand{\goodfrac}{\alpha}

\newcommand{\sN}{\mathcal{N}}
\renewcommand{\Var}{\operatorname{Var}}
\newcommand{\bP}{\mathbb{P}}
\newcommand{\Csep}{C_{\mathrm{sep}}}

\def\[#1\]{\begin{align}#1\end{align}}
\def\(#1\){\begin{align*}#1\end{align*}}

\newcommand{\bprf}{\begin{proof}}
\newcommand{\eprf}{\end{proof}}
\newcommand{\blem}{\begin{lemma}}
\newcommand{\elem}{\end{lemma}}

\newcommand{\psos}{\preceq_{\mathrm{sos}}}
\newcommand{\ssos}{\succeq_{\mathrm{sos}}}
\DeclareMathOperator{\vc}{vec}

\newcommand{\op}{\mathrm{op}}

\newcommand{\sT}{\mathcal{T}}
\newcommand{\sI}{\mathcal{I}}

\newcommand{\sos}{\mathrm{sos}}
\newcommand{\citep}[1]{\cite{#1}}
\newcommand{\citet}[1]{\cite{#1}}

\newcommand{\Poincare}{Poincar\'{e} }
\title{Better Agnostic Clustering via Relaxed Tensor Norms}

\author{Pravesh K. Kothari \and Jacob Steinhardt}

\begin{document}

\maketitle
 \draftbox
\thispagestyle{empty}


We develop a new family of convex relaxations for $k$-means clustering based on sum-of-squares 
norms, a relaxation of the injective tensor norm that is efficiently computable using the 
Sum-of-Squares algorithm. We give an algorithm based on this relaxation that recovers a 
faithful approximation to the true means in the given data whenever the low-degree moments of 
the points in each cluster have bounded sum-of-squares norms.

We then prove a sharp upper bound on the sum-of-squares norms for moment tensors of any distribution 
that satisfies the \emph{\Poincare inequality}. The \Poincare inequality is a central 
inequality in probability theory, and a large class of distributions satisfy it including 
Gaussians, product distributions, strongly log-concave distributions, and any sum or 
uniformly continuous transformation of such distributions.

As an immediate corollary, for any $\gamma > 0$, we obtain an efficient algorithm for learning 
the means of a mixture of $k$ arbitrary \Poincare distributions in $\R^d$ in time $d^{O(1/\gamma)}$ so long as the means have separation $\Omega(k^{\gamma})$. This in particular yields an algorithm for learning Gaussian 
mixtures with separation $\Omega(k^{\gamma})$, thus partially resolving an open problem 
of Regev and Vijayaraghavan \citet{regev2017learning}.

Our algorithm works even in the robust setting where an $\epsilon$ fraction of 
arbitrary outliers are added to the data, as long as the fraction of outliers 
is smaller than the smallest cluster. We therefore obtain results in the strong 
agnostic setting where, in addition to not knowing the distribution family, 
the data itself may be arbitrarily corrupted.


\clearpage

\ifnum\showtableofcontents=1
{
\tableofcontents
\thispagestyle{empty}
 }
\fi

\clearpage

\setcounter{page}{1}

\section{Introduction}
\label{sec:intro}
Progress on many fundamental unsupervised learning tasks has required circumventing a plethora of 
intractability results by coming up with natural restrictions on input instances 
that preserve some essential character of the problem. For example, while $k$-means clustering 
is NP-hard in the worst-case \citep{mahajan2009planar}, there is an influential line of work 
providing spectral algorithms for clustering mixture models satisfying appropriate assumptions
\citep{achlioptas2005spectral,kumar2010clustering,awasthi2012improved}. 
On the flip side, we run the risk of 
developing algorithmic strategies that exploit strong assumptions in a way that makes them 
brittle. We are thus forced to walk the tight rope of avoiding computational intractability 
without ``overfiting'' our algorithmic strategies to idealized assumptions on input data. 

Consider, for example, the problem of clustering data into $k$ groups. On the one hand, 
a line of work leading to \citep{awasthi2012improved} shows that a variant of spectral 
clustering can recover the underlying clustering so long as each cluster has 
bounded covariance around its center and the cluster centers are separated by at least 
$\Omega(\sqrt{k})$. Known results can improve on this bound to require a separation of  
$\Omega(k^{1/4})$ if the cluster distributions are assumed to be isotropic and log-concave 
\citep{vempala2002spectral}. If the cluster means are in general position, other lines of work 
yields results for Gaussians 
\citep{kalai2010efficiently,moitra2010settling,belkin2010polynomial,hsu13spherical,
bhaskara2014uniqueness,goyal2014fourier,bhaskara2014smoothed,anderson2014blessing,ge2015learning} 
or for distributions satisfying independence assumptions \citep{hsu09spectral,anandkumar13tensor}. 
However, the assumptions often play a crucial role in the algorithm. For example, the famous 
method of moments that yields a result for learning mixtures of Gaussians in general position 
uses the specific algebraic structure of the moment tensor of Gaussian distributions. Such 
techniques are unlikely to work for more general classes of distributions.  

As another example, consider the robust mean estimation problem which has been actively 
investigated recently. 
Lai et. al. \cite{lai2016agnostic} and later improvements 
\citep{diakonikolas2017practical,steinhardt2018resilience} show how to estimate the mean of 
an unknown distribution (with bounded second moments) 
where an $\epsilon$ fraction of points are adversarially corrupted, obtaining 
additive error $O(\sqrt{\epsilon})$. On the other hand, 
Diakonikolas et. al. \cite{diakonikolas2016robust} showed how to 
estimate the mean of a Gaussian or product distribution 
with nearly optimal 
additive error $\tilde{O}(\epsilon)$. However, their algorithm again makes strong use of 
the known algebraic structure of the moments of these distributions. 

Further scrutiny reveals that the two examples of clustering and robust mean estimation 
suffer from a ``second-moment'' barrier. 
For both problems, the most general results algorithmically exploit only some boundedness 
condition on the second moments of the data, while the strongest results use \emph{exact} 
information about higher moments (e.g. by assuming Gaussianity) and are thus brittle. 
This leads to the key conceptual driving force of the present work:

\begin{center} \emph{Can we algorithmically exploit boundedness information about a limited number of low-degree moments?} \end{center}

As the above examples illustrate, this is a natural way to formulate the ``in-between'' 
case between the two well-explored extremes. From an algorithmic perspective, this question 
forces us to develop techniques that can utilize information about higher moments of data for 
problems such as clustering and mean estimation. For these problems, we can more concretely ask:

\begin{center}
\emph{Can we beat the second-moment barrier in the agnostic setting for clustering and robust mean estimation?}
\end{center}

The term \emph{agnostic} here refers to the fact that we want our algorithm to work for as wide a class of distributions as possible, and in particular to avoid making parametric assumptions (such as Gaussianity) about the underlying distribution.

The main goal of this work is to present a principled way to utilize higher moment information in input data and break the second moment barrier for both clustering and robust mean estimation. A key primitive in our approach is algorithmic certificates upper bounding the injective norms of \emph{moment tensors} of data. 

Given input points, consider the injective tensor norm of their moments that generalizes 
the spectral norm of a matrix:
\begin{equation}
\label{eq:moment-tensor}
\sup_{\|v\|_2 \leq 1} \frac{1}{n} \sum_{i=1}^n \langle x_i, v \rangle^{2t}.
\end{equation}
For $t > 1$, bounds on the injective norm of the moment tensor present a natural way to 
utilize higher moment information in the given data, which suggests an 
avenue for algorithm design. Indeed, one of our contributions (Theorem~\ref{thm:main-recover}) 
is a generalization of spectral norm clustering that uses estimates of injective norms of moment tensors to go beyond the second moment barrier.

Unfortunately for us, estimating injective norms (unlike the spectral norm) is intractable. While it is likely easier than computing injective norms for arbitrary tensors, it turns out that approximately computing injective norms for moment tensors is equivalent to the well-studied problem of approximating the $2 \to q$ norm which is known to be small-set-expansion hard \cite{MR2961513-Barak12}. The best known algorithms for approximating $2 \to q$ norm achieve a multiplicative approximation ratio of $d^{\Theta(q)}$ in $d$ dimensions, and while known hardness results \cite{MR2961513-Barak12} only rule out some fixed constant factor algorithms for this problem, it seems likely 
that there is no polynomial time algorithm for $2 \to q$ norm that achieves \emph{any} 
dimension-independent approximation ratio.
 
An average-case variant of approximating injective norms of moment tensors has been studied to some extent due to its relationship to the small-set-expansion problem. The sum-of-squares hierarchy of semi-definite programming 
relaxations turns out to be a natural candidate algorithm in this setting and is known to exactly compute the injective norm in specialized settings such as that of the Gaussian distribution. On the other hand, the most general such results \cite{MR2961513-Barak12,MR3388192-Barak15} imply useful bounds only for settings similar to product distributions . 


One of the key technical contributions of this work is to go beyond product distributions for 
estimating injective norms. Specifically, we show (Theorem~\ref{thm:main-cert}) 
that Sum-of-Squares gives a polynomial time procedure to show a dimension-free upper bound on the
injective norms of (large enough i.i.d.~samples from) arbitrary distributions that satisfy a \emph{Poincar\'e inequality}. 
This is a much more satisfying state of affairs as it immediately captures all strongly log-concave 
distributions, including correlated Gaussians. Further, the Poincar\'e inequality is 
robust---i.e., it continues to hold under uniformly continuous transformations of the underlying 
space, as well as bounded re-weightings of the probability density. 

Without further ado, we define Poincar\'e distributions:
A distribution $p$ on $\bR^d$ is said to be \emph{$\sigma$-Poincar\'e} if for all differentiable 
functions $f : \bR^d \to \bR$ we have
\begin{equation}
\label{eq:poincare-intro}
\Var_{x \sim p}[f(x)] \leq \sigma^2 \bE_{x \sim p}[\|\nabla f(x)\|_2^2].
\end{equation}
This is a type of isoperimetric inequality on the distribution $x$ and implies concentration of measure. In 
Section~\ref{sec:poincare} we discuss in more detail various examples of distributions that 
satisfy \eqref{eq:poincare-intro}, as well as properties of such distributions. \Poincare inequalities and distributions are intensely studied in probability theory; indeed, we rely on one such powerful result of Adamczak and Wolff \citet{adamczak2015concentration} for establishing a sharp bound on the sum-of-squares algorithm's estimate of the injective norm of an i.i.d.~sample from a \Poincare distribution. 


We then confirm the intuitive claim that understanding injective norms of moment tensors 
can give us an algorithmic tool to beat the second moment barrier, by combining our result 
on certification of \Poincare distributions with our algorithm for clustering under such certificates.
 Specifically, we show that for any $\gamma > 0$, 
given a balanced mixture of $k$ Poincar\'e distributions with means separated by 
$\Omega(k^{\gamma})$, we can successfully cluster $n$ samples from this mixture 
in $n^{\oo(1/\gamma)}$ time 
(by using $\oo(1/\gamma)$ levels of the sum-of-squares hierarchy). Similarly, given samples 
from a Poincar\'e distribution with an $\epsilon$ fraction of adversarial corruptions, we can 
estimate its mean up to an error of $\oo(\epsilon^{1-\gamma})$ in $n^{\oo(1/\gamma)}$ time. In fact, we will see below that we get both at once: 
a robust clustering algorithm that can learn well-separated mixtures even in the presence of 
arbitrary outliers.

To our knowledge such a result was not previously known even in the 
second-moment case (\citet{charikar2017learning} and \citet{steinhardt2018resilience} 
study this setting but only obtain results in the \emph{list-decodable} learning model). 
Our result only relies on the SOS-certifiability of the moment tensor, 
and holds for any deterministic point set for which such a sum-of-squares certificate exists.

Despite their generality, our results are strong enough to yield new bounds 
even in very specific settings
such as  learning balanced mixtures of $k$ spherical Gaussians with separation $\Omega(k^{\gamma})$. Our algorithm allows recovering the true means in $n^{O(1/\gamma)}$ time and 
partially resolves an open problem posed in the recent work of \citet{regev2017learning}. 

Certifying injective norms of moment tensors appears to be a useful primitive 
and could help enable further applications of the sum of squares method in machine learning. Indeed, \citep{kothari2017outlier} studies the problem of robust estimation of higher 
moments of distributions that satisfy a bounded-moment condition closely related to approximating injective norms. Their relaxation and the analysis are significantly different from the present work; nevertheless, our result for 
\Poincare distributions immediately implies that the robust moment estimation algorithm of 
\citep{kothari2017outlier} succeeds for a large class of \Poincare distributions.



\subsection{Main Results and Applications}

Our first main result regards efficient upper bounds on the injective norm of 
the moment tensor of any \Poincare distribution. 
Let $x_1, x_2, \ldots, x_n \in \bR^d$ be $n$ i.i.d.~samples from a \Poincare distribution with mean $\mu$,
and let $M_{2t} = \frac{1}{n}\sum_{i = 1}^n (x_i-\mu)^{\otimes 2t}$ be the empirical estimate of the $2t$th moment tensor. 
We are interested in upper-bounding the injective norm \eqref{eq:moment-tensor}, which can be 
equivalently expressed in terms of the moment tensor as 
\begin{equation}
\label{eq:moment-tensor-2}
\sup_{\|v\|_2 \leq 1} \frac{1}{n} \sum_{i=1}^n \langle x_i - \mu, v \rangle^{2t} = \sup_{\|v\|_2 \leq 1} \iprod{M_{2t}, v^{\otimes 2t}}.
\end{equation}
Standard results (see Fact \ref{fact:tail-bound-poincare}) yield dimension-free upper bounds on 
\eqref{eq:moment-tensor-2} for all \Poincare distributions. 
Our first result is a ``sum-of-squares proof'' of this fact giving an \emph{efficient} 
method to certify dimension-free upper bounds on \eqref{eq:moment-tensor-2} for samples 
from any \Poincare distribution.

Specifically, let the {sum of squares} norm of $M_{2t}$, denoted by $\|M_{2t}\|_{\sos_{2t}}$, 
be the degree-$2t$ sum-of-squares relaxation of \eqref{eq:moment-tensor-2} 
(we discuss such norms and the sum-of-squares method in more detail in Section \ref{sec:prelims}; for now the 
important fact is that $\|M_{2t}\|_{\sos_{2t}}$ can be computed in time $(nd)^{\oo(t)}$).  
 We show that for a large enough sample from a distribution that satisfies the 
\Poincare inequality, the sum-of-squares norm of the moment tensor is upper bounded by a dimension-free constant.

\begin{theorem}
\label{thm:main-cert}
Let $p$ be a $\sigma$-\Poincare distribution over $\bR^d$ with mean $\mu$. 
Let $x_1, \ldots, x_n \sim p$ with $n \geq (2d\log(dt/\delta))^t$. Then, 
for some constant $C_t$ (depending only on $t$) with probability at least $1-\delta$ 
we have $\|M_{2t}\|_{\sos_{2t}} \leq C_t \sigma$, where 
$M_{2t} = \frac{1}{n} \sum_{i=1}^n (x_i - \mu)^{\otimes 2t}$. 
\end{theorem}
As noted above, previous sum-of-squares bounds worked for specialized cases such as product distributions. 
Theorem~\ref{thm:main-cert} is key to our applications that crucially rely on 
1) going beyond product distributions and 2) using $\sos_{2t}$ norms as a proxy for injective norms for higher moment tensors. 

\paragraph{Outlier-Robust Agnostic Clustering} 
Our second main result is an efficient algorithm for \emph{outlier-robust agnostic 
clustering} whenever the ``ground-truth'' clusters have moment tensors with bounded sum-of-squares norms.  

Concretely, the input is data points $x_1, \ldots, x_n$ of $n$ points in $\bR^d$, a $(1-\epsilon)$ fraction 
of which admit a (unknown) partition into sets $I_1, \ldots, I_k$ each having bounded sum-of-squares norm 
around their corresponding means $\mu_1, \ldots, \mu_k$. The remaining $\epsilon$ fraction can be arbitrary outliers. 
Observe that in this setting, we do not make any explicit distributional assumptions. 

We will be able to obtain strong estimation guarantees in this setting so long 
as the clusters are \emph{well-separated} 
and the fraction $\epsilon$ of outliers is not more than $\alpha/8$, 
where $\alpha$ is the fraction of points in the smallest cluster. 
We define the separation as $\Delta = \min_{i \neq j} \|\mu_i - \mu_j\|_2$. 
A lower bound on $\Delta$ is information theoretically necessary even 
in the special case of learning mixtures of identity-covariance gaussians 
without any outliers (see \cite{regev2017learning}). 
\newcommand{\Out}{\mathsf{out}}
\begin{theorem}
\label{thm:main-recover}
Suppose points $x_1, \ldots, x_n \in \bR^d$ can be partitioned into sets 
$I_1, \ldots, I_k$ and  $\Out$, where the $I_j$ are the clusters and $\Out$ is 
a set of outliers of size $\epsilon n$. Suppose $I_j$ has size $\alpha_j n$ and mean $\mu_j$, 
and that its $2t$th moment $M_{2t}(I_j)$ satisfies $\|M_{2t}(I_j)\|_{\sos_{2t}} \leq B$. 
Also suppose that 
$\epsilon \leq \alpha/8$ for $\alpha = \min_{j=1}^k \alpha_j$.

Finally, suppose the separation $\Delta \geq \Csep \cdot B / \alpha^{1/t}$, with $\Csep \geq C_0$ (for a universal 
constant $C_0$).
Then there is an algorithm running in time $(nd)^{\oo(t)}$ and outputting means 
$\hat{\mu}_1, \ldots, \hat{\mu}_k$ such that 
$\|\hat{\mu}_j - \mu_j\|_2 \leq \oo(B(\epsilon/\alpha + \Csep^{-2t})^{1-1/2t})$ for all $j$.
\end{theorem}

The parameter $B$ specifies a bound on the variation in each cluster. The separation 
condition says that the distance between cluster means must be slightly larger 
(by a $\alpha^{-1/t}$ factor) than this variation. 
The error in recovering the cluster means depends on two terms---the fraction of 
outliers $\epsilon$, and the separation $\Csep$. 

To understand the guarantees of the theorem, let's start with the case where 
$\epsilon = 0$ (no outliers) and $\alpha = 1/k$ (all clusters have the same size). 
In this case, the separation requirement between the clusters is $B \cdot k^{1/t}$ 
where $B$ is the bound on the moment tensor of order $2t$. 
The theorem guarantees a recovery of the means up to an error in Euclidean norm of $\oo(B)$. 
By taking $t$ larger (and spending the correspondingly larger running time), 
our clustering algorithm works with separation $k^{\gamma}$ for any constant $\gamma$. 
This is the first result that goes beyond the separation requirement of $k^{1/2}$ in the 
\emph{agnostic} clustering setting---i.e., without making distributional assumptions on the clusters.

It is important to note that even in 1 dimension, it is information theoretically impossible to 
recover cluster means to an error $\ll B$ when relying only on $2t$th moment bounds.  
A simple example to illustrate this is obtained by taking a mixture of two distributions 
on the real line with bounded $2t$th moments but small overlap in the tails. In this case, it is 
impossible to correctly classify the points that come from the the overlapping part. Thus, a 
fraction of points in the tail always end up misclassified, shifting the true means. The recovery 
error of our algorithm does indeed drop as the separation (controlled by $\Csep$) between the 
true means increases (making the overlapping parts of the tail smaller). 
We note that for the specific case of spherical gaussians, we can exploit their parametric 
structure to get arbitrarily accurate estimates even for fixed separation; see Corollary \ref{cor:gaussian-mixture}.

Next, let's consider $\epsilon \neq 0.$ In this case, if $\epsilon \ll \alpha$, we recover the means up to an error of $\oo(B)$ again (for $\Csep \geq C_0$).
It is intuitive that the recovery error for the means 
should grow with the number of outliers, and the condition $\epsilon \leq \alpha/8$ 
is necessary, as if $\epsilon \geq \alpha$ then the outliers could form an entirely 
new cluster making recovery of the means information-theoretically impossible.


We also note that in the degenerate case where $k = 1$ (a single cluster), 
Theorem~\ref{thm:main-recover} yields results for robust mean estimation of a 
set of points corrupted by an $\epsilon$ fraction of outliers. 
In this case we are able to estimate the mean to error 
$\epsilon^{\frac{2t-1}{2t}}$; when $t = 1$ this is $\sqrt{\epsilon}$, which matches 
the error obtained by methods based on second moments 
\citep{lai2016agnostic,diakonikolas2017practical,steinhardt2018resilience}. 
For $t = 2$ we get error $\epsilon^{3/4}$, 
for $t=3$ we get error $\epsilon^{5/6}$, and so on, approaching an error of 
$\epsilon$ as $t \to \infty$. In particular, this pleasingly approaches the rate 
$\tilde{\oo}(\epsilon)$ obtained by much more bespoke methods that rely strongly 
on specific distributional assumptions \citep{lai2016agnostic,diakonikolas2016robust}.

Note that we could not hope to do better than $\epsilon^{\frac{2t-1}{2t}}$, 
as that is the information-theoretically optimal error for distributions with bounded 
$2t$th moments (even in one dimension), and degree-$2t$ SOS only ``knows about'' moments 
up to $2t$.

Finally, we can obtain results even for clusters that are 
not well-separated, and for fractions of outliers that could exceed $\alpha$. 
In this case we no longer output exactly $k$ means, and must instead consider the 
list-decodable model \citet{balcan2008discriminative,charikar2017learning}, 
where we output a list of $\oo(1/\alpha)$ means of which the true 
means are a sublist. We defer the statement of this result to 
Theorem~\ref{thm:robust-clustering-alpha} in Section~\ref{sec:clustering}.

\paragraph{Applications}
Putting together Theorem~\ref{thm:main-cert} and Theorem~\ref{thm:main-recover} 
immediately yields corollaries for learning mixtures of 
Poincar\'e distributions, and in particular mixtures of Gaussians. 

\begin{corollary}[Disentangling Mixtures of Arbitrary \Poincare Distributions]
\label{cor:poincare-mixture}
Suppose that we are given a dataset of $n$ points $x_1, \ldots, x_n$, 
such that at least $(1-\epsilon)n$ points are drawn from a mixture 
$\alpha_1p_1 + \cdots + \alpha_kp_k$ of $k$ distributions, where 
$p_j$ is $\sigma$-Poincar\'e with mean $\mu_j$ (the remaining $\epsilon n$ 
points may be arbitrary). Let $\alpha = \min_{j=1}^k \alpha_j$. 
Also suppose that the separation $\Delta$ is at least 
$\Csep \cdot C_t \sigma / \alpha^{1/t}$, for some constant $C_t$ 
depending only on $t$ and some $\Csep \geq 1$.

Then, assuming that $\epsilon \leq \frac{\alpha}{10}$, 
for some $n = \oo((2d\log(tkd/\delta))^t/\alpha + d\log(k/\delta)/\alpha\epsilon^2)$, 
there is an algorithm running in $n^{\oo(t)}$ time which with probability 
$1-\delta$ outputs 
candidate means $\hat{\mu}_1, \ldots, \hat{\mu}_k$ such that 
$\|\hat{\mu}_j - \mu_j\|_2 \leq C_t' \sigma (\epsilon / \alpha + \Csep^{-2t})^{\frac{2t-1}{2t}}$ 
for all $j$ (where $C_t'$ is a different universal constant).
\end{corollary}
The $1/\alpha$ factor in the sample complexity is so that we have enough samples from 
every single cluster for Theorem~\ref{thm:main-cert} to hold. The extra term of 
$d\log(k/\delta)/\epsilon^2$ in the sample complexity is so that the empirical 
means of each cluster concentrate to the true means.

Corollary~\ref{cor:poincare-mixture} is one of the strongest results on learning 
mixtures that one could hope for. If the mixture weights $\alpha$ are all at least 
$1/\poly(k)$, then Corollary~\ref{cor:poincare-mixture} implies that we can cluster 
the points as long as the separation $\Delta = \Omega(k^{\gamma})$ for any $\gamma > 0$. 
Even for spherical Gaussians the best previously known algorithms required separation 
$\Omega(k^{1/4})$. On the other hand, Corollary~\ref{cor:poincare-mixture} applies 
to a large family of distributions including arbitrary strongly log-concave distributions. 
Moreover, while the \Poincare inequality does not directly hold for discrete distributions, 
Fact~\ref{thm:poincare-bounded} in Section~\ref{sec:poincare} implies that a large 
class of discrete distributions, including product distributions over bounded domains, 
will satisfy the Poincar\'e inequality after adding zero-mean Gaussian noise.
Corollary~\ref{cor:poincare-mixture} then yields a clustering algorithm 
for these distributions, as well.

For mixtures of Gaussians in particular, we can do better, and in fact achieve vanishing error independent 
of the separation:
\begin{corollary}[Learning Mixtures of Gaussians]
\label{cor:gaussian-mixture}
Suppose that $x_1, \ldots, x_n \in \bR^d$ are drawn from a mixture of 
$k$ Gaussians: $p = \sum_{j=1}^k \alpha_j \sN(\mu_j, I)$, where 
$\alpha_j \geq 1/\poly(k)$ for all $j$. 
Then for any $\gamma> 0$, there is a separation $\Delta_0 = \oo(k^{\gamma})$ 
such that given $n \geq \poly(d^{1/\gamma}, k, 1/\epsilon)\log(k/\delta)$ samples from  
$p$, if the separation $\Delta \geq \Delta_0$, then with probability $1-\delta$ 
we obtain estimates 
$\hat{\mu}_1, \ldots, \hat{\mu}_k$ with $\|\hat{\mu}_j - \mu_j\|_2 \leq \epsilon$ for all $j$. 
\end{corollary}
\begin{remark}
This partially resolves an open question of \citet{regev2017learning}, who ask whether 
it is possible to efficiently learn mixtures of Gaussians with separation $\sqrt{\log k}$.
\end{remark}

The error now goes to $0$ as $n \to \infty$, which is 
not true in the more general Corollary~\ref{cor:poincare-mixture}. This requires invoking
Theorem IV.1 of \citet{regev2017learning}, which, given a sufficiently good 
initial estimate of the means of a mixture of Gaussians, shows how to get an arbitrarily accurate estimate.
As discussed before, such a result is specific to Gaussians and in particular 
is information-theoretically impossible for mixtures of general Poincar\'e distributions.

\subsection{Proof Sketch and Technical Contributions}

We next sketch the proofs of our two main theorems (Theorem~\ref{thm:main-cert} 
and Theorem~\ref{thm:main-recover}) while indicating which parts involve new technical ideas.

\subsubsection{Sketch of Theorem~\ref{thm:main-cert}}
\label{sec:sketch-cert}

For simplicity, we will only focus on SOS-certifiability in the infinite-data limit, 
i.e. on showing that SOS can certify an upper bound 
$\bE_{x \sim p}[\langle x-\mu, v \rangle^{2t}] \leq C_t\sigma^{2t}\|v\|_2^{2t}$. 
(In Section~\ref{sec:t-general} we will show that finite-sample concentration follows 
due to the matrix Rosenthal inequality \citep{mackey2014matrix}.)

We make extensive use of a result of \citet{adamczak2015concentration}; it is a very general 
result on bounding non-Lipschitz functions of Poincar\'e distributions, but in our context 
the important consequence is the following: 
\begin{quote}
If $f(x)$ 
is a degree-$t$ polynomial such that $\bE_p[\nabla^j f(x)] = 0$ for 
$j = 0, \ldots, t-1$, then $\bE_p[f(x)^2] \leq C_t\sigma^{2t} \|\nabla^t f(x)\|_F^2$ for 
a constant $C_t$, assuming $p$ is $\sigma$-Poincar\'e. 
(Note that $\nabla^t f(x)$ is a constant since $f$ is degree-$t$.)
\end{quote}
Here $\|A\|_F^2$ denotes the Frobenius norm of the tensor $A$, 
i.e. the $\ell_2$-norm of $A$ if it were flattened into a $d^{t}$-element vector.

We can already see why this sort of bound might be useful for $t = 1$. Then if we let 
$f_v(x) = \langle x - \mu, v \rangle$, we have $\bE[f_v(x)] = 0$
and hence $\bE_p[\langle x-\mu, v \rangle^2] \leq C_1\sigma^2 \|v\|_2^2$. This exactly 
says that $p$ has bounded covariance.

More interesting is the case $t = 2$. Here we will let 
$f_A(x) = \langle (x-\mu)(x-\mu)^{\top} - \Sigma, A \rangle$, where $\mu$ is the mean 
and $\Sigma$ is the covariance of $p$. It is easy to see that both 
$\bE[f_A(x)] = 0$ and $\bE[\nabla f_A(x)] = 0$. Therefore, we have 
$\bE[\langle (x-\mu)(x-\mu)^{\top} - \Sigma, A \rangle^2] \leq C_2\sigma^4 \|A\|_F^2$.

Why is this bound useful? It says that if we unroll $(x-\mu)(x-\mu)^{\top} - \Sigma$ to a 
$d^2$-dimensional vector, then this vector has bounded covariance (since if we project 
along any direction $A$ with $\|A\|_F = 1$, the variance is at most $C_2\sigma^4$).
This is useful because it turns out sum-of-squares ``knows about'' such covariance bounds; 
indeed, this type of covariance bound is exactly the property used in 
\citet{barak2012hypercontractivity} to certify 
$4$th moment tensors over the hypercube. In our case it yields a sum-of-squares proof that 
$\bE[\langle (x-\mu)^{\otimes 4} - \Sigma^{\otimes 2}, v^{\otimes 4} \rangle] \psos C_2\sigma^4 \|v\|_2^4$, which can then be used to bound the $4$th moment 
$\bE[\langle x-\mu, v \rangle^4] = \bE[\langle (x-\mu)^{\otimes 4}, v^{\otimes 4} \rangle]$.

Motivated by this, it is natural to try the same idea of ``subtracting off the 
mean and squaring'' with $t = 4$. Perhaps we could define $f_A(x) = \langle ((x-\mu)^{\otimes 2} - \Sigma)^{\otimes 2} - \bE[((x-\mu)^{\otimes 2} - \Sigma)^{\otimes 2}], A \rangle$? 

Alas, this does not work---while there is a suitable polynomial $f_A(x)$ for $t=4$ that yields 
sum-of-squares bounds, it is somewhat 
more subtle. For simplicity we will write the polynomial for $t=3$. It is the
following: $f_A(x) = \langle (x-\mu)^{\otimes 3} - 3(x-\mu) \otimes \Sigma  - M_3, A \rangle$, 
where $M_3 = \bE[(x-\mu)^{\otimes 3}]$ is the third-moment tensor of $p$.
By checking that $\bE[f_A(x)] = \bE[\nabla f_A(x)] = \bE[\nabla^2 f_A(x)] = 0$, we obtain 
that the tensor $F_3(x) = (x-\mu)^{\otimes 3} - 3(x-\mu) \otimes \Sigma - M_3$, when unrolled 
to a $d^3$-dimensional vector, has bounded covariance, which means that sum-of-squares 
knows that $\bE[\langle F_3(x)^{\otimes 2}, v^{\otimes 6} \rangle]$ is bounded for all 
$\|v\|_2 \leq 1$.

However, this is not quite what we want---we wanted to show that 
$\bE[\langle (x-\mu)^{\otimes 6}, v^{\otimes 6} \rangle]$ is bounded. Fortunately, 
the leading term of $F_3(x)^{\otimes 2}$ is indeed $(x-\mu)^{\otimes 6}$, and all the 
remaining terms are lower-order. So, we can subtract off $F_3(x)$ and recursively bound 
all of the lower-order terms to get a sum-of-squares bound on 
$\bE[\langle (x-\mu)^{\otimes 6}, v^{\otimes 6} \rangle]$. The case of general $t$ 
follows similarly, by carefully constructing a tensor $F_t(x)$ whose first $t-1$ 
derivatives are all zero in expectation.

There are a couple contributions here beyond what was known before. The first is 
identifying appropriate tensors $F_t(x)$ whose covariances are actually bounded so that 
sum-of-squares can make use of them. For $t = 1, 2$ (the cases that had previously 
been studied) the appropriate tensor is in some sense the ``obvious'' one 
$(x-\mu)^{\otimes 2} - \Sigma$, but even 
for $t = 3$ we end up with the fairly non-obvious tensor 
$(x-\mu)^{\otimes 3} - 3(x-\mu) \otimes \Sigma - M_3$. 
(For $t=4$ it is $(x-\mu)^{\otimes 4} - 6(x-\mu)^{\otimes 2} \otimes \Sigma - 4(x-\mu) \otimes M_3 - M_4 + 6\Sigma \otimes \Sigma$.)
While these tensors may seem mysterious a priori, they are actually the unique tensor 
polynomials with leading term $x^{\otimes t}$ such that all derivatives of order $j < t$ 
have mean zero.
Even beyond Poincar\'e distributions, these seem like useful building blocks for 
sum-of-squares proofs.

The second contribution is making the connection between Poincar\'e distributions 
and the above polynomial inequalities. The well known work of Lata\l{}a 
\citet{latala2006estimates} establishes non-trivial estimates of upper bounds on the 
moments of polynomials of Gaussians, of which the inequalities used here are a special case. 
\citet{adamczak2015concentration} show that these inequalities 
also hold for Poincar\'e distributions. However, it is not a priori obvious that these 
inequalities should lead to sum-of-squares proofs, and it requires a careful invocation 
of the general inequalities to get the desired results in the present setting.

\subsubsection{Sketch of Theorem~\ref{thm:main-recover}}

We next establish our result on robust clustering. In fact we will establish a 
robust mean estimation result which will lead to the clustering result---specifically, 
we will show that if a set of points $x_1, \ldots, x_n$ contains a subset 
$\{x_i\}_{i \in I}$ of size $\alpha n$ that is SOS-certifiable, then the mean 
(of the points in $I$) can be estimated regardless of the remaining points. 
There are two parts: if $\alpha \approx 1$ we want to show error going to $0$ as 
$\alpha \to 1$, while if $\alpha \ll 1$ we want to show error that does not grow too fast 
as $\alpha \to 0$. In the latter case we will output $\oo(1/\alpha)$ candidates for the mean 
and show that at least one of them is close to the true mean (think of these candidates 
as accounting for $\oo(1/\alpha)$ possible clusters in the data). We will later 
prune down to exactly $k$ means for well-separated clusters.

For $t = 1$ (which corresponds to bounded covariance), the $\alpha \to 0$ case is studied 
in \citet{charikar2017learning}. A careful analysis of the proof there reveals that all of 
the relevant inequalities are sum-of-squares inequalities, so there is a 
sum-of-squares generalization of the algorithm in \citet{charikar2017learning} that should 
give bounds for SOS-certifiable distributions. While this would likely lead to some robust 
clutering result, we note the bounds we achieve here are stronger than those in 
\citet{charikar2017learning}, as \citet{charikar2017learning} do not achieve tight results 
when the clusters are well-separated.
Moreover, the proof in \citet{charikar2017learning} is complex and would be somewhat tedious to 
extend in full to the sum-of-squares setting.


We combine and simplify ideas from both 
\citet{charikar2017learning} and \citet{steinhardt2018resilience} to obtain a relatively 
clean algorithm. In fact, 
we will see that a certain mysterious constraint appearing in \citet{charikar2017learning} is 
actually the natural constraint from a sum-of-squares perspective.

Our algorithm is based on the following optimization. Given 
points $x_1, \ldots, x_n$, we will try to find points $w_1, \ldots, w_n$ such that 
$\frac{1}{n} \sum_{i=1}^n \pE_{\xi(v)}[\langle x_i - w_i, v \rangle^{2t}]$ 
is small for all pseudodistributions $\xi$ over the sphere.
This is natural because we know that for the good points $x_i$ and the true mean $\mu$, 
$\langle x_i - \mu, v \rangle^{2t}$ is small (by the SOS-certifiability assumption). 
However, without further constraints this is not a very good idea because the trivial 
optimum is to set $w_i = x_i$. We would somehow like to ensure that the $w_i$ cannot overfit 
too much to the $x_i$; it turns out that the natural way to measure this 
degree of overfitting is via the quantity $\sum_{i \in I} \langle w_i - \mu, w_i \rangle^{2t}$. 

Of course, this quantity is not known because we do not know $\mu$. But we 
\emph{do} know that $\sum_{i \in I} \pE_{{\xi}(v)}[\langle w_i - \mu, v \rangle^{2t}]$ 
is small for all pseudodistributions 
(because the corresponding quantity is small for $x_i - \mu$ and $w_i - x_i$, and 
hence also for $w_i - \mu = (w_i - x_i) + (x_i - \mu)$ by Minkowski's inequality).
Therefore, we impose the following constraint: 
\emph{whenever $z_1, \ldots, z_n$ are such that 
$\sum_{i=1}^n \pE_{{\xi}}[\langle z_i, v \rangle^{2t}] \leq 1$ for all $\xi$, 
it is also the case that 
$\sum_{i=1}^n \langle z_i, w_i \rangle^{2t}$ is small}. This constraint is not efficiently 
imposable, but it does have a simple sum-of-squares relaxation. Namely, we require that 
$\sum_{i=1}^n \langle Z_i, w_i^{\otimes 2t} \rangle$ is small whenever 
$Z_1, \ldots, Z_n$ are pseudomoment tensors satisfying $\sum_{i=1}^n Z_i \psos I$.

Together, this leads to seeking $w_1, \ldots, w_n$ such that
\begin{equation}
\sum_{i=1}^n \pE_{\xi}[\langle x_i - w_i, v \rangle^{2t}] \text{ is small for all $\xi$, and } \sum_{i=1}^n \langle Z_i, w_i^{\otimes 2t} \rangle \text{ is small whenever $\sum_i Z_i \psos I$.}
\end{equation}
If we succeed in this, we can show that we end up with a good estimate of the mean 
(more specifically, the $w_i$ are clustered into a small number of clusters, such that 
one of them is centered near $\mu$). The above is a convex program, and thus,  if this is impossible, by duality 
there must exist specific $\xi$ and $Z_1, \ldots, Z_n$ such that the above 
quantities cannot be small for \emph{any} $w_1, \ldots, w_n$. But for fixed 
$\xi$ and $Z_{1:n}$, the different $w_i$ are independent of each other, and 
in particular it should be possible to make both sums small at least for the terms coming 
from the good set $I$. This gives us a way of performing \emph{outlier removal}: look for 
terms where $\min_{w} \pE_{{\xi}}[\langle x_i - w, v \rangle^{2t}]$ or 
$\min_{w} \langle Z_i, w \rangle$ is large, and remove those from the set of points. We 
can show that after a finite number of iterations this will have successfully removed 
many outliers and few good points, so that eventually we must 
succeed in making both sums small and thus get a successful clustering.

Up to this point the proof structure is similar to 
\citet{steinhardt2018resilience}; the main innovation is the constraint 
involving the $z_i$, which bounds the degree of overfitting. 
In fact, when $t=1$ this constraint is the dual form 
of one appearing in \citet{charikar2017learning}, which asks that $w_i^{\otimes 2} \preceq Y$ 
for all $i$, for some matrix $Y$ of small trace. In \citet{charikar2017learning}, 
the matrix $Y$ couples all of the variables, which complicates the analysis. 
In the form given here, we avoid the coupling and also see why the constraint is the 
natural one for controlling overfitting.

To finish the proof, it is also necessary to iteratively re-cluster the $w_i$ and 
re-run the algorithm on each cluster. This is due to issues where we might have, say, $3$ 
clusters, where the first two are relatively close together but very far from the third one. 
In this case our algorithm would resolve the third cluster from the first two, but needs 
to be run a second time to then resolve the first two clusters from each other.

\citet{charikar2017learning} also use this re-clustering idea, but their re-clustering 
algorithm makes use of a sophisticated metric embedding technique and is relatively complex. 
Here we avoid this complexity by making use of \emph{resilient sets}, an idea introduced 
in \citet{steinhardt2018resilience}. A resilient set is a set such that all large subsets have mean 
close to the mean of the original set; it can be shown that any set with bounded moment tensor 
is resilient, and by finding such resilient sets we can robustly cluster in a much more direct 
manner than before. In particular, in the well-separated case we show that after enough 
rounds of re-clustering, every resilient 
set has almost all of its points coming from a single cluster, leading to substantially 
improved error bounds in that case.

\subsection{Open Problems}

In this work, we showed that sum-of-squares can certify moment 
tensors for distributions satisfying the \Poincare inequality. While 
this class of distributions is fairly broad, one could hope to establish 
sum-of-squares bounds for even broader families. Indeed, one canonical 
family is the class of \emph{sub-Gaussian} distributions. Is it the case 
that sum-of-squares certifies moment tensors for all sub-Gaussian distributions? 
Conversely, are there sub-Gaussian distributions that sum-of-squares cannot certify?
Even for $4$th moments, this is unknown:

\begin{openquestion}
Let $p$ be a $\sigma$-sub-Gaussian distribution and let $M_4(p)$ denote its 
fourth moment tensor. Is it always the case that $\|M_4(p)\|_{\sos_{2t}} \leq \oo(1) \cdot \sigma$ 
for some constant $t$?
\end{openquestion}

In another direction, the only property we required from \Poincare distributions 
is Adamczak and Wolff's result \citep{adamczak2015concentration} 
bounding the variance of polynomials whose derivatives all have mean $0$.
Adamczak and Wolff show that this property also holds for other distributions, 
such as sub-Gaussian product distributions. 
One might expect additional distributions to satisfy these inequalities as well, 
in which case our present results would apply unchanged.
 
\begin{openquestion}
Say that a distribution $p$ satisfies the $(t,\sigma)$-moment property 
if, whenever $f(x)$ is a degree-$t$ polynomial with $\bE[f(x)] = \bE[\nabla f(x)] = \cdots = \bE[\nabla^{t-1}f(x)] = 0$, 
we have $\bE[f(x)^2] \leq \sigma^{2t} \|\nabla^t f(x)\|_F^2$. 
Which distributions satisfy the $(t,\sigma)$-moment property?
\end{openquestion}

Finally, the present results all regard certifying moment tensors in the $\ell_2$-norm, 
i.e., on upper bounding $\langle M_{2t}(p), v^{\otimes 2t} \rangle$ for all $\|v\|_2 \leq 1$. 
However, \citet{steinhardt2018resilience} show that in some cases--such as discrete distribution 
learning--the $\ell_{\infty}$-norm is more natural. To this end, 
define $\|M_{2t}\|_{\sos_{2t},\infty}$ to be the maximum of 
$\pE_{\xi(v)}[\langle M_{2t}, v^{\otimes 2t} \rangle]$ over all pseudodistributions on the \emph{hypercube}.

\begin{openquestion}
For what distributions $p$ is $\|M_{2t}\|_{\sos_{2t},\infty}$ small?
Additionally, do bounds on $\|M_{2t}\|_{\sos_{2t},\infty}$ lead to better 
robust estimation and clustering in the $\ell_{\infty}$-norm?
\end{openquestion}

\section{Preliminaries}
\label{sec:prelims}
In this section we set up notation and introduce a number of preliminaries regarding 
sum-of-squares algorithms. 

\paragraph{Notation}
We will use $d$ to denote dimension, and $n$ the number of samples 
in a dataset $x_1, \ldots, x_n$. For clustering problems $k$ will denote the number of 
clusters. $\epsilon$ will denote, depending on circumstance, either the desired estimation 
error or the fraction of adversarial corruptions for a robust estimation problem. 
$\delta$ will denote the probability of failure of an algorithm. 
$\gamma$ will denote an exponent which we think of as going to zero, as in phrases like 
``$\oo(k^{\gamma})$ for any $\gamma > 0$''. For tensors, $t$ will denote their order (or $2t$ 
if we want to emphasize the order is even). We let $C$ denote a universal constant and $C_t$ a 
universal constant depending on $t$ (these constants may change in each place they are used).

Below we use Theorem (and Proposition, Lemma, etc.) for results that we prove in this paper, 
and Fact for results proved in other papers.

\paragraph{Tensors, Polynomials and Norms}
A $t$th order tensor $T$ on $\R^d$ is a $t$-dimensional array of real numbers indexed by $t$-tuples on $[d].$ 
$T$ is naturally associated with a homogenous degree $t$ polynomial $T(v) = \langle T, v^{\otimes t} \rangle$. 
The \emph{injective norm} of a tensor $T$ is defined as $\sup_{\|v\|_2 \leq 1} T(v).$

Given a distribution $p$ on $\R^d$, the $t$th \emph{moment tensor} of $p$ is defined by $M_t(p) = \E_{x \sim p} [(x-\mu)^{\otimes t}]$, where $\mu$ is the mean of $p$. 
Observe that each entry of $M_t(p)$ is the expectation of some monomial of degree $t$ with 
respect to $p.$
For a finite set of points $S$, we let $\E_{x \sim S}$ denote expectation with 
respect to its empirical distribution. The moment tensor of a set of points is the moment 
tensor of its empirical distribution.

Given a matrix $M$, we let $\|M\|_{\op}$ denote its operator norm (maximum singular value) 
and $\|M\|_F$ denotes its Frobenius norm ($\ell_2$-norm of its entries when flattened to 
a $d^2$-dimensional vector). More generally, for a tensor $T$, we let 
$\|T\|_F$ denote the Frobenius norm (which is again the $\ell_2$-norm when the entries 
are flattened to a $d^t$-dimensional vector).

\subsection{Sum-of-Squares Programs and Pseudodistributions}

In this paper we are interested in approximating injective norms of moment tensors, 
i.e. of upper-bounding programs of the form
\begin{align}
\label{eq:injective-prelim}
\text{maximize} &\ \frac{1}{n} \sum_{i=1}^n \langle x_i, v \rangle^{2t} \\
\text{subject to} &\ \|v\|_2 = 1.
\end{align}
This problem is hard to solve exactly, so we will instead consider the following 
\emph{sum-of-squares relaxation} of \eqref{eq:injective-prelim}:
\begin{align}
\label{eq:relax-prelim}
\text{maximize} &\ \frac{1}{n} \sum_{i=1}^n \pE_{\xi(v)}[\langle x_i, v \rangle^{2t}] \\
\label{eq:sphere} \text{subject to} &\ \pE_{\xi(v)}[(\|v\|_2^2-1)p(v)] = 0 \text{ for all polynomials } p(v) \text{ of degree at most $2t-2$}, \\
\label{eq:sos} &\ \pE_{\xi(v)}[q(v)^2] \geq 0 \text{ for all polynomials } q(v) \text{ of degree at most $t$}, \\
\label{eq:norm} &\ \pE_{\xi(v)}[1] = 1.
\end{align}
Formally, $\pE_{\xi(v)}$, which we refer to as a \emph{pseudo-expectation}, is simply a 
linear functional on the space of polynomials in $v$ of degree at most $2t$. The last 
two constraints say that $\xi$ specifies a \emph{pseudo-distribution}, 
meaning that it respects all properties of a regular probability distribution that can 
be specified with low-degree polynomials. Meanwhile, the first constraint is a relaxation 
of the constraint that $\|v\|_2 = 1$. If $\xi$ satisfies (\ref{eq:sphere}-\ref{eq:norm}), 
we say that $\xi$ is a \emph{pseudo-distribution on the sphere}.

The relaxation \eqref{eq:relax-prelim} is a special case of the well-studied family of 
sum-of-squares relaxations. It is well-known that these programs can be solved efficiently, 
i.e. in time $(nd)^{\oo(t)}$, due to the ability to represent them as semidefinite programs 
\citep{MR939596-Shor87,parrilo2000structured,MR1748764-Nesterov00,MR1846160-Lasserre01}.

The key strategy for bounding the value of \eqref{eq:relax-prelim} is 
\emph{sum-of-squares proofs}. We say that a polynomial inequality 
$p(v) \leq q(v)$ has a sum-of-squares proof if the polynomial $q(v) - p(v)$ can be 
written as a sum of squares of polynomials. We write this as $p(v) \psos q(v)$, or 
$p(v) \preceq_{2t} q(v)$ if we want to emphasize that the proof only involves polynomials 
of degree at most $2t$.

For pseudo-distributions on the sphere, we will extend this notation and say 
that $p(v) \preceq_{2t} q(v)$ if $q(v) - p(v) - r(v)(\|v\|_2^2-1)$ is a sum of squares 
for some polynomial $r(v)$ of degree at most $2t-2$. 

Now, let $p_0(v) = \frac{1}{n} \sum_{i=1}^n \langle x_i, v \rangle^{2t}$. 
Imagine that there is some sequence of sum-of-squares proofs 
$p_0(v) \preceq_{2t} p_1(v) \preceq_{2t} \cdots \preceq_{2t} p_m(v) = B^{2t} \cdot \|v\|_2^{2t}$. 
Then we know, by the constraints on $\xi$, 
that $\pE_{\xi(v)}[p_0(v)] \leq B^{2t}\pE_{\xi(v)}[\|v\|_2^{2t}] = B^{2t}$. Therefore, 
such a proof immediately implies that \eqref{eq:relax-prelim} has value at most $B^{2t}$. 

Note that since the relation $\preceq_{2t}$ is transitive, this is equivalent to 
the condition $p_0(v) \preceq_{2t} B^{2t}\|v\|_2^{2t}$. 
In this case, we call the set of points \emph{SOS-certifiable}. More generally, for 
a distribution we have the following definition:
\begin{definition}
For a distribution $p$, we say that $p$ is \emph{$(2t,B)$-SOS-certifiable} if 
$\langle M_{2t}(p), v^{\otimes 2t} \rangle \preceq_{2t} B^{2t}\|v\|_2^{2t}$. 
We will alternately denote this by $\|M_{2t}(p)\|_{\sos_{2t}} \leq B$.
For a set of points $x_1, \ldots, x_n$, we say it is 
$(2t,B)$-SOS-certifiable if its empirical distribution is SOS-certifiable.
\end{definition}
Note that $\langle M_{2t}(p), v^{\otimes 2t} \rangle = \bE_{x \sim p}[\langle x, v \rangle^{2t}]$ 
so that this definition coincides with certifying \eqref{eq:relax-prelim}.

\subsection{Basic Sum-of-Squares Facts}

We capture a few basic facts about sum-of-squares proofs that we will use later. 
First, sum-of-squares can certify all spectral norm bounds:
\begin{fact}[Sum-of-Squares Proofs from Spectral Norm Bounds]
For any symmetric $d \times d$ matrix $M$, $\langle M, x^{\otimes 2} \rangle \preceq_{2} \|M\|_{\op} \|x\|_2^2.$
\label{fact:spectral-sos}
\end{fact}
As a corollary (by applying the spectral norm bound to higher-order tensors), we 
also have the following:
\begin{fact}[Spectral Norm Bounds for Tensors]
\label{fact:spectral}
For a degree-$t$ tensor function $F(x)$, suppose that 
$\bE_{x \sim p}[\langle F(x), A \rangle^2] \leq \lambda \|A\|_F^2$ 
for all symmetric tensors $A$. Then, 
$\bE_{x \sim p}[\langle F(x), v^{\otimes t}\rangle^2] \preceq_{2t} \lambda \|v\|_2^{2t}$.
\end{fact}
Fact~\ref{fact:spectral} is used crucially in the sequel. In fact, all of our 
sum-of-squares proofs will essentially involve invoking 
Fact~\ref{fact:spectral} on appropriately chosen tensors $F(x)$.

Finally, the following basic inequality holds:
\begin{fact}
\label{fact:mult}
If $0 \preceq_t p(v) \preceq_{t} p'(v)$ and $0 \preceq_s q(v) \preceq_s q'(v)$, 
then $p(v)q(v) \preceq_{s+t} p'(v)q'(v)$.
\end{fact}
This is useful because it allows us to multiply sum-of-squares inequalities 
together.




%
%

\section{\Poincare Distributions}
\label{sec:poincare}
In this section, we note some important properties of the class of all \Poincare distributions. 
\begin{definition}[\Poincare Distributions]
\label{def:poincare}
A distribution $p$ over $\bR^d$ is said to be $\sigma$-\Poincare if it satisfies the following Poincar\'{e} inequality with parameter 
$\sigma$: 
For all differentiable functions $f : \bR^d \to \bR$, 
\begin{equation}
\label{eq:poincare-s3}
\Var_{x \sim p}[f(x)] \leq \sigma^2 \bE_{x \sim p}[\|\nabla f(x)\|_2^2].
\end{equation}

\end{definition}

Note that no discrete distribution can satisfy \eqref{eq:poincare-s3}.
To see why, consider for instance the uniform distribution over $\{0,1\}$. 
This cannot satisfy \eqref{eq:poincare-s3} for any $\sigma$, because for 
the function $f(x) = \max(0, \min(1, 3x-2))$ we have $\Var[f(x)] = \frac{1}{4}$ 
but $\bE[\|\nabla f(x)\|_2^2] = 0$. More generally, \eqref{eq:poincare-s3} implies 
that there are no low-probability ``valleys'' separating two high-probability regions.

Next we give some examples of distributions satisfying 
\eqref{eq:poincare-s3}. First, if $p = \sN(\mu, \Sigma)$ is a normal 
distribution with mean $\mu$ and variance $\Sigma$, then $p$ satisfies 
\eqref{eq:poincare-s3} with parameter $\sigma^2 = \|\Sigma\|_{\op}$. 
More generally, any \emph{strongly log-concave} distribution satisfies the 
Poincar\'{e} inequality:
\begin{fact}[\citet{bakry1985diffusions}]
\label{fact:log-concave}
Suppose that $p(x) = \exp(-\psi(x))$, where $\psi$ is a function with 
strictly positive curvature: $\nabla^2 \psi(x) \succeq \frac{1}{\sigma^2}I$ for 
all $x \in \bR^d$. Then $p$ satisfies the Poincar\'{e} inequality with parameter 
$\sigma$.
\end{fact}
We recover the Gaussian case by letting 
$\psi(x) = \frac{1}{2}(x-\mu)^{\top}\Sigma^{-1}(x-\mu)$.
Fact~\ref{fact:log-concave} was first established by Bakry and \'{E}mery 
\citet{bakry1985diffusions} who actually prove a stronger 
\emph{log-Sobolev} inequality for $p$.

Another important class is the family of distributions with bounded support. It cannot be 
the case that an arbitrary bounded distribution satisfies \eqref{eq:poincare-s3}, because 
we have already seen that no discrete distribution can satisfy \eqref{eq:poincare-s3}. 
However, it is always possible to add noise to the distribution that smooths out the 
support and allows \eqref{eq:poincare-s3} to hold. Specifically:
\begin{fact}[\citet{bardet2018functional}]
\label{thm:poincare-bounded}
Suppose that $p$ is a distribution on $\bR^d$ whose support has radius at most 
$R$ in $\ell_2$-norm. 
Let $p_{\tau}$ denote the result of adding Gaussian noise with variance $\tau^2 I$ to $p$. 
Then, if $\tau \geq 2R$, $p_{\tau}$ satisfies the Poincar\'{e} inequality 
with parameter $\tau\sqrt{e}$.
\end{fact}
So, we can always cause a bounded random variable to satisfy \eqref{eq:poincare-s3} 
by adding sufficiently large Gaussian noise. We remark that while this is very 
useful in low dimensions, in high dimensions the radius $R$ typically grows as 
$\sqrt{d}$, in which case Fact~\ref{thm:poincare-bounded} does not give very good 
bounds.

\paragraph{Composition rules} The Poincar\'{e} inequality is also preserved under 
products, sums, and uniformly continuous transformations. Specifically, we have the 
following:
\begin{proposition}
\label{thm:poincare-composition}
The following composition rules hold:
\begin{itemize}
\item If independent random variables $X_1$ and $X_2$ are 
      $\sigma_1$- and $\sigma_2$-Poincar\'e, respectively, then 
      the product distribution $(X_1, X_2)$ is 
      $\max(\sigma_1, \sigma_2)$-Poincar\'e.
\item If independent random variables $X_1$ and $X_2$ are 
      $\sigma_1$- and $\sigma_2$-Poincar\'e, respectively, then 
      $aX_1 + bX_2$ is $\sqrt{a^2\sigma_1^2 + b^2\sigma_2^2}$-Poincar\'e.
\item If $X$ is $\sigma$-Poincar\'e and 
      $\Phi : \bR^d \to \bR^{d'}$ satisfies $\|\Phi(x') - \Phi(x)\|_2 \leq L\|x' - x\|_2$ 
      for all $x, x' \in \bR^d$, then $\Phi(X)$ is $(L\sigma)$-Poincar\'e.
\end{itemize}
\end{proposition}
\noindent The above properties are all straightforward, but for completeness we 
prove them in Appendix~\ref{sec:poincare-composition-proof}.

\paragraph{Implications of the Poincar\'{e} inequality} 
We end with some implications of \eqref{eq:poincare-s3}, which in particular specify 
properties that any distribution satisfying \eqref{eq:poincare-s3} must have.

First, any distribution satisfying the Poincar\'{e} inequality has 
exponentially decaying tails. Specifically, the following is a well known fact:
\begin{fact}[Tail Bound for Poincare Distributions]  \label{fact:tail-bound-poincare}
For any unit vector $v$,  and $X$ a $\sigma$-\Poincare random variable with mean $\mu$, 
we have
\[
\bP[|\langle v,x - \mu \rangle| \geq z] \leq 6\exp(-z/\sigma).
\]
\end{fact}

More generally, any distribution satisfying the Poincar\'{e} inequality also 
satisfies a Lipschitz concentration property:
\begin{fact}[\citet{bobkov1997poincare}]
\label{thm:lipschitz}
If $p$ satisfies the Poincar\'{e} inequality with parameter $\sigma$ and 
$f : \bR^d \to \bR$ satisfies $\|\nabla f(x)\|_2 \leq L$ for all $x$, 
then $\bE[\exp(f(x) - \bE[f(x)])] \leq \frac{2 + \sigma L}{2 - \sigma L}$.
In particular, $\bP[|f(x) - \bE[f(x)]| \geq z] \leq 6\exp(-\frac{z}{\sigma L})$.
\end{fact}
Exponential concentration of Lipschitz functions (with weaker bounds than above) was first 
observed in \citet{gromov1983topological}.
Fact~\ref{thm:lipschitz} generalizes the previous point 
on exponential concentration of linear functions, as can be seen by taking $f(x) = \langle v, x \rangle$.

Finally, and most important to our subsequent analysis, the gradient bound 
\eqref{eq:poincare-s3} implies an analogous bound for higher-order derivatives 
as well. Specifically, \eqref{eq:poincare-s3} can be re-written as 
saying that $\bE[f(x)^2] \leq \sigma^2 \bE[\|\nabla f(x)\|_2^2]$ 
whenever $\bE[f(x)] = 0$. More generally, we have:
\begin{fact}[\citet{adamczak2015concentration}]
\label{thm:adamczak}
Suppose that $p$ satisfies the Poincar\'{e} inequality with 
parameter $\sigma$. Then, if $f : \bR^d \to \bR$ is a function satisfying 
$\bE_{x \sim p}[f(x)] = \bE_{x \sim p}[\nabla f(x)] = \cdots = \bE_{x \sim p}[\nabla^{t-1}f(x)] = 0$, we have
\begin{equation}
\label{eq:poincare-higher-order}
\bE_{x \sim p}[f(x)^2] \leq C_t \sigma^{2t} \bE_{x \sim p}[\|\nabla^t f(x)\|_F^2],
\end{equation}
where $C_t$ is a constant depending only on $t$.
\end{fact}
For $t=1$, we recover the usual Poincar\'{e} inequality. 
Thus Fact~\ref{thm:adamczak} can be interpreted as saying that if 
\eqref{eq:poincare-higher-order} holds for $t=1$, it holds for all $t > 1$ as well.

Despite its simplicity, Fact~\ref{thm:adamczak} is actually a 
highly non-trivial consequence of the Poincar\'{e} inequality.
It is a special case of results due to Adamczak and Wolff
\citet{adamczak2015concentration} (see Theorem 3.3 and the ensuing discussion therein); 
those results in turn build on work of Lata\l{}a \citet{latala2006estimates}. 

In the next section, we will use Fact~\ref{thm:adamczak} to obtain a low-degree sum-of-squares proof of a sharp upper bound on the injective norms of moment tensors of arbitrary \Poincare distributions. 

\section{Certifying Injective Norms for Poincare Distributions}

Fact \ref{fact:tail-bound-poincare} shows an upper bound on 
$\bE[\exp(\frac{1}{\sigma}\langle x, v \rangle)]$ for $\sigma$-Poincar\'e distributions, and 
hence in particular on the $2t$th moments 
$\sup_{\|v\|_2 \leq 1} \bE[\langle x, v\rangle^{2t}]$ for any $t$. 
The goal of this section is to show a low-degree sum-of-squares proof of this fact. 

Specifically, the main goal is to show the following:
\begin{theorem} \label{thm:certifying-injective-norm-main} \label{thm:cert}
Let $p$ be a zero-mean, $\sigma$-\Poincare distribution with $2t$-th 
moment tensor 
$M_{2t}(p) = \bE_{x \sim p}[(x-\mu)^{\otimes 2t}]$ , where $\mu = \bE_{x \sim p}[x]$ 
is the mean of $p$. Then, for all $t$, $M_{2t}(p)$ is $(2t,C_t \sigma)$-SOS-certifiable:
\begin{equation} \label{eq:cert}
 \langle M_{2t}(p), v^{\otimes 2t} \rangle \preceq_{2t} C_t \sigma^{2t} \|v\|_2^{2t} \text{ for some universal constant } C_t.
\end{equation}
Moreover, given $n \geq (2d\log(td/\delta))^t$ samples from $p$, 
with probability $1-\delta$ the moment tensor $\hat{M}_{2t}$ of the empirical 
distribution will also satisfy \eqref{eq:cert} (with a different constant $C_t$).
\end{theorem}
Recall that $p(v) \preceq_{2t} q(v)$ means that there is a degree-$2t$ sum-of-squares proof 
that $q(v) - p(v) \geq 0$ (as a polynomial in $v$).

In the rest of this section, we will first warm up by proving 
Theorem~\ref{thm:cert} for $t = 1,2$; the proof in these cases is standard for sum-of-squares 
experts, but will help to illustrate a few important ideas used in the sequel.
Next, in Section~\ref{sec:t-3} we will prove Theorem~\ref{thm:cert} for $t = 3$, 
which is no longer standard and contains most of the ideas in the general case. 
Finally, we will prove the general case in Section~\ref{sec:t-general}.
We leave the issue of finite-sample concentration to the very end.

For $t = 1$, Theorem~\ref{thm:certifying-injective-norm-main} has a simple proof: 
$\langle M_2, v^{\otimes 2} \rangle \preceq_{2} \|M_2\|_{\op} \|v\|_2^2 \preceq_{2} \oo(\sigma^2) \|v\|_2^2$ using Fact \ref{fact:spectral-sos} and the fact that $\sigma$-Poincar\'e 
distributions have covariance $\oo(\sigma^2)$. This simply asserts that degree-$2$ sum-of-squares 
knows about spectral norm upper bounds.

For $t=2$, a natural idea is to flatten $M_4$ (which is a $d \times d \times d \times d$ tensor) 
into a $d^2 \times d^2$ matrix $F(M_4)$, and obtain upper bounds in terms of the 
spectral norm of this flattened matrix.

\renewcommand{\vc}{\operatorname{vec}}
Unfortunately, even for a Gaussian distribution this estimate can be off by a factor as 
large as the dimension $d$. Specifically, if $p = \sN(0, \Sigma)$ is a Gaussian with 
variance $\Sigma$, then the flattening $F(M_4)$ is equal to 
$\vc(\Sigma)\vc(\Sigma)^{\top} + 2\Sigma \otimes \Sigma$, where 
$\vc(\Sigma)$ flattens the $d$-dimensional matrix $\Sigma$ to a 
$d^2$-dimensional vector, and $(A \otimes B)_{ii',jj'} = A_{ij}B_{i'j'}$. 
(We skip the standard argument based on Isserlis' theorem.)
This is problematic because it means that 
$\|F(M_4)\|_{\op}^2 \geq \|\vc(\Sigma)\|_2^2 = \|\Sigma\|_F^2$. If e.g. $\Sigma = I_d$ is 
the identity matrix, then $\|\Sigma\|_F^2 = d$, while we would hope to certify an upper bound 
of $\oo(1)$.

%

The key idea that allows us to get a (much) improved bound here is to observe that, as 
polynomials in $v$, 
$\langle v^{\otimes 2}, \vc(\Sigma)\vc(\Sigma)^{\top} v^{\otimes 2} \rangle = \langle v^{\otimes 2}, (\Sigma \otimes \Sigma) v^{\otimes 2} \rangle$. 
That is, the degree-$4$ polynomials defined by $\vc(\Sigma)\vc(\Sigma)^{\top}$ and 
$\Sigma \otimes \Sigma$ are equal---this allows us to ``change the representation'' for the 
same polynomial to go to a ``representation'' where the associated matrix has a smaller 
spectral norm. This fact has a simple sum-of-squares proof (it is sometimes referred to by 
saying that pseudo-distributions respect PPT symmetries) and allows us to now upper bound 
\begin{align}
\langle v^{\otimes 2}, F(M_4) v^{\otimes 2} \rangle
 &= \langle v^{\otimes 2}, (\vc(\Sigma)\vc(\Sigma)^{\top} + 2\Sigma \otimes \Sigma) v^{\otimes 2} \rangle \\
 &= 3 \langle v^{\otimes 2}, (\Sigma \otimes \Sigma) v^{\otimes 2} \rangle \\
 &\preceq_4 3 \|\Sigma \otimes \Sigma\|_{\op}^2 \|v^{\otimes 2}\|_2^2 \\
 &= 3\|\Sigma\|_{\op}^4 \|v\|_2^4.
\end{align}

The above argument shows how one can exploit the symmetry properties of pseudo-distributions in 
order to certify strikingly better upper bounds on the maximum of the degree-$4$ polynomials 
associated with the moment tensors. This suggests a natural strategy for going beyond 
$t = 2$: write the moment tensor as a sum of (constantly many) terms and 
show that each term, as a polynomial, is equivalent to one where the canonical flattening as a 
matrix has a small spectral norm. This argument can (with much tedium) be made to work for 
small $t$s but can get unwieldy for large $t$s. However, the issue with 
the above argument is that it uses that the structure of the moment tensor was known 
to us. In our argument for arbitrary \Poincare distributions, we cannot rely 
on knowing the structure of the moment tensors and so will need a different proof 
technique. 

\paragraph{Degree-$2$ proof for Poincar\'e distributions}
To establish sum-of-squares bounds for general Poincar\'e distributions, we make use 
of Fact~\ref{thm:adamczak} from Section~\ref{sec:poincare}. Recall that Fact~\ref{thm:adamczak} 
states (for $t=2$) that if $f(x)$ satisfies $\bE[f(x)] = \bE[\nabla f(x)] = 0$, 
then $\bE[f(x)^2] \leq C \cdot \sigma^4 \cdot \bE[\|\nabla^2 f(x)\|_F^2]$.
We will define the polynomial $f_A(x) = \langle (x-\mu)(x-\mu)^{\top} - \Sigma, A \rangle$, 
where $\mu$ is the mean and $\Sigma = M_2(p)$ is the covariance of $p$. 

Note that $\bE[f_A(x)] = \bE[\nabla f_A(x)] = 0$, while 
$\bE[\|\nabla^2 f_A(x)\|_F^2] = \bE[\|2A\|_F^2] = 4\|A\|_F^2$. Therefore, 
we have $\bE[\langle (x-\mu)(x-\mu)^{\top} - \Sigma, A \rangle^2] \leq \oo(\sigma^4)\|A\|_F^2$ 
for all matrices $A$. This implies that for the tensor 
$F_2(x) = (x-\mu)(x-\mu)^{\top} - \Sigma$, we have 
$\langle \bE[F_2(x)^{\otimes 2}], v^{\otimes 4} \rangle 
 = \bE[\langle F_2(x), v^{\otimes 2} \rangle^2] \preceq_4 \oo(\sigma^4)\|v\|_2^4$ 
(by Fact~\ref{fact:spectral}). 

Next note that $\bE[F_2(x)^{\otimes 2}] = M_4 - \Sigma \otimes \Sigma$. Therefore, we have 
the sum-of-squares proof
\begin{align}
\langle M_4, v^{\otimes 4} \rangle
 &= \langle M_4 - \Sigma \otimes \Sigma, v^{\otimes 4} \rangle + \langle \Sigma \otimes \Sigma, v^{\otimes 4} \rangle \\
 &\preceq_4 \oo(\sigma^4)\|v\|_2^4 + \langle v^{\otimes 2}, (\Sigma \otimes \Sigma) v^{\otimes 2} \rangle \\
 &\preceq_4 \oo(\sigma^4)\|v\|_2^4 + \|\Sigma\|_{\op}^2 \|v\|_2^4 \\
 &= \oo(\sigma^4) \|v\|_2^4,
\end{align}
yielding the desired certificate for Poincar\'e distributions for $t = 2$.

The crucial step was the bound 
$\bE[\langle F_2(x), A \rangle^2] = \bE[\langle (x-\mu)(x-\mu)^{\top} - \Sigma, A \rangle^2] \leq \oo(\sigma^4)\|A\|_F^2$, which bounds the covariance of $F_2(x)$ (when considered as a vector) and hence 
yields spectral information which is accessible to sum-of-squares via 
Fact~\ref{fact:spectral}. This inequality holds for 
$\sigma$-Poincar\'e distributions but \emph{not} for arbitrary distributions with bounded 
$4$th moments. 
The case of general $t > 2$ involves conjuring similar operator norm bounds for other polynomials 
in $x$ which combine to yield bounds on the moment tensor $M_{2t}$. We will see this 
next for the case of $t=3$.






\subsection{The case of $t =3$}
\label{sec:t-3}

We now handle the case of $t=3$, assuming that we already know a bound for $t=2$. (We will 
formally induct on $t$ in the next section.) Specifically, in this section we will show:
\begin{proposition} \label{prop:warmup-deg6}
Suppose that $p$ is $\sigma$-Poincar\'e and we are given that 
$\langle M_2, v^{\otimes 2} \rangle \preceq_2 \oo(1) \cdot \sigma^2 \|v\|_2^2$ and 
$\langle M_4, v^{\otimes 4} \rangle \preceq_{4} \oo(1) \cdot \sigma^4 \|v\|_2^4.$ 
Then, $\langle M_6, v^{\otimes 6} \rangle \preceq_6 \oo(1) \cdot \sigma^6 \|v\|_2^6.$
\end{proposition}

The key idea in the argument is to observe that Adamczak and Wolff's generalized \Poincare 
inequality (Fact~\ref{thm:adamczak}) implies a sum of squares upper bound on the 
injective norm of a tensor related to $M_6$.

\begin{lemma} \label{lem:deg-6-adamczak}
Let $p$ be a $\sigma$-\Poincare probability distribution. For any tensor $A$ of order $3$, 
let $f_A(x) = \langle (x-\mu)^{\otimes 3} - 3(x-\mu) \otimes M_2 - M_3, A \rangle$. Then, $\E_x[f_A^2(x)] \leq \oo(1) \cdot \sigma^6 \|A\|_F^2$. 
\end{lemma}

\begin{proof}
The following is a direct computation:
\begin{align}
\nabla f_A(x) &= \langle 3(x-\mu)^{\otimes 2} \otimes I - 3M_2 \otimes I, A \rangle, &&\text{ and hence } \bE[\nabla f_A(x)] = 0 \\
\nabla^2 f_A(x) &= \langle 6(x-\mu) \otimes I \otimes I, A \rangle, &&\text{ and hence } \bE[\nabla^2 f_A(x)] = 0 \\
\nabla^3 f_A(x) &= \langle 6I \otimes I \otimes I, A \rangle = 6A, &&\text{ and hence } \bE[\|\nabla ^3 f_A(x)\|_F^2] = 36\|A\|_F^2.
\end{align}
Therefore, by Fact~\ref{thm:adamczak}, we have $\bE[f_A(x)^2] \leq \oo(1) \cdot \sigma^6 \|A\|_F^2$ for all tensors $A$ of order $3$. 
\end{proof}

%

Define $F_3(x) = (x-\mu)^{\otimes 3} - 3(x-\mu) \otimes M_2 - M_3$. 
The conclusion of Lemma~\ref{lem:deg-6-adamczak} together with 
Fact~\ref{fact:spectral} implies that 
$\langle \bE[F_3(x)^{\otimes 2}], v^{\otimes 6} \rangle = \bE[\langle F_3(x)^{\otimes 2}, v^{\otimes 3} \rangle^2] \preceq_6 \oo(1) \cdot \sigma^6 \|v\|_2^6$.

To complete the proof, we will write $\langle M_6, v^{\otimes 6} \rangle $ (the polynomial 
we are trying to bound) as a linear combination of 
$\langle \bE[F_3(x)^{\otimes 2}], v^{\otimes 6} \rangle$ together with other terms 
that we can bound by appealing to upper bounds on $\langle M_4, v^{\otimes 4} \rangle$ 
and $\langle M_2, v^{\otimes 2} \rangle$. 

\begin{proof}[Proof of Proposition \ref{prop:warmup-deg6}]
First we compute $\bE[F_3(x)^{\otimes 2}]$:
\begin{equation}
\label{eq:F3-expansion}
\bE[F_3(x)^{\otimes 2}] = M_6 + 9M_2 \otimes M_2 \otimes M_2 - 6M_4 \otimes M_2 - M_3 \otimes M_3.
\end{equation}
Since we know that $\langle \bE[F_3(x)^{\otimes 2}], v^{\otimes 6} \rangle \preceq_6 \oo(1) \cdot \sigma^6 \|v\|_2^6$, 
\eqref{eq:F3-expansion} implies that 
\begin{align}
\langle M_6, v^{\otimes 6} \rangle 
 &= \langle \bE[F_3(x)^{\otimes 2}], v^{\otimes 6} \rangle + \langle M_6 - \bE[F_3^{\otimes 2}], v^{\otimes 6} \rangle \\
 &\preceq_6 \oo(1) \cdot \sigma^6 \|v\|_2^6 + 6\langle M_4, v^{\otimes 4}\rangle \langle M_2, v^{\otimes 2} \rangle + \langle M_3, v^{\otimes 3} \rangle^2 - 9\langle M_2, v^{\otimes 2}\rangle^3.
\end{align}
The final term $-9\langle M_2, v^{\otimes 2}\rangle^3$ is a negative sum of squares, so we 
can ignore it. 
On the other hand, we have
\begin{align}
6\langle M_4, v^{\otimes 4} \rangle\langle M_2, v^{\otimes 2} \rangle + \langle M_3, v^{\otimes 3} \rangle^2 
 &\stackrel{(i)}{\preceq_6} 7\langle M_4, v^{\otimes 4} \rangle \langle M_2, v^{\otimes 2} \rangle \\
 &\stackrel{(ii)}{\preceq_6} 7 \cdot (\oo(1) \cdot \sigma^4 \|v\|_2^4) \cdot (\oo(1)\cdot \sigma^2 \|v\|_2^2) = \oo(1) \cdot \sigma^6 \|v\|_2^6.
\end{align}
Here we used two facts: (i) $\langle M_3, v^{\otimes 3} \rangle^2 \preceq_6 \langle M_2, v^{\otimes 2}\rangle\langle M_4, v^{\otimes 4}\rangle$ (by Cauchy-Schwarz) 
and (ii) Fact~\ref{fact:mult} on multiplying together sum-of-squares inequalities.

As we have bounded all terms by $\oo(1) \cdot \sigma^6 \|v\|_2^6$, the proof is complete.
\end{proof}

\subsection{The case of arbitrary $t$}
\label{sec:t-general}

We now generalize the argument from the previous section to arbitrary $t$s. 

As in the case of $t=3$ above, we will first come up with a polynomial of degree $2t$ that we 
can upper bound directly using the generalized \Poincare inequality. For this purpose, we 
define the order-$t$ tensor $F_t(x)$ generalizing $F_3$ from the previous section.

Let $\sT_t$ denote the set of all tuples $(i_0, i_1, \ldots, i_r)$ of integers such that $i_0 \geq 0$ and 
$i_s \geq 2$ for $s > 0$, and $i_0 + \cdots + i_r = t$. We will take 
\begin{equation}
\label{eq:Ft-def}
F_t(x) \eqdef \sum_{(i_0, \ldots, i_r) \in \sT_k} (-1)^r \binom{k}{i_0 \ \cdots \ i_r} (x - \mu)^{\otimes i_0} \otimes M_{i_1} \otimes \cdots \otimes M_{i_r}.
\end{equation}
While \eqref{eq:Ft-def} may seem mysterious, $F_t(x)$ is in fact the unique tensor such 
that $\nabla^t F_t(x) = I$ and $\bE[\nabla^s F_t(x)] = 0$ for $s < t$.

We verify the latter property in the following Lemma~\ref{lem:zero} 
(we omit the proof of uniqueness).
For any order-$t$ tensor $A$, $F_t$ defines a polynomial $f_A(x) = \langle F_t(x), A \rangle.$ 
We will show that the partial derivatives of $f_A$ (w.r.t $x$) of all orders 
$0,1, \ldots, t-1$ have expectation zero, so that we can apply the higher-order Poincar\'e 
inequality from Fact~\ref{thm:adamczak}. 
\begin{lemma}
\label{lem:zero}
For every order-$t$ $A$ and for every $j = 0, \ldots, t-1$, $\bE[\langle \nabla^j F_t(x), A \rangle] = 0$.
\end{lemma}
\begin{proof}
We show the following structural property of the partial derivative tensors of $F$ for every $x$:
\begin{claim}
\label{claim:partial-derivs}
$\langle \nabla F_t(x), A \rangle = t \langle F_{t-1}(x)\otimes I, A \rangle $ for every symmetric tensor $A$. 
\end{claim}
\begin{proof}[Proof of Claim]
We have 
\begin{align}
\nabla \langle F_t(x) , A \rangle 
 &= \sum_{(i_0, \ldots, i_r) \in \sT_t} (-1)^r \binom{t}{i_0 \ \cdots \ i_r} \langle \nabla (x-\mu)^{\otimes i_0} \otimes M_{i_1} \otimes \cdots \otimes M_{i_r}, A \rangle  \\
 &= \sum_{(i_0, \ldots, i_r) \in \sT_t} (-1)^r i_0 \binom{k}{i_0 \ \cdots \ i_r} \langle (x-\mu)^{\otimes (i_0-1)} \otimes I \otimes M_{i_1} \otimes \cdots \otimes M_{i_r}, A \rangle \\
 &\stackrel{(i)}{=} t\sum_{(i_0-1, \ldots, i_r) \in \sT_{t-1}} (-1)^r \binom{k}{i_0-1 \ \cdots \ i_r} \langle (x-\mu)^{\otimes (i_0-1)} \otimes I \otimes M_{i_1} \otimes \cdots \otimes M_{i_r}, A \rangle \\
 &\stackrel{(ii)}{=} t \langle F_{t-1}(x) \otimes I, A \rangle,
\end{align}
In (i), we fold the factor of $i_0$ into the multinomial coefficient, and in (ii) 
we use the symmetries of $A$ to permute tensor modes.
\end{proof}
By Claim~\ref{claim:partial-derivs}, it suffices to show that 
$\bE_x[ \langle F_t(x), A \rangle ] = 0$ for $t > 0$ (since the derivates of $F_t$ are 
simply of the form $F_s$ for $s < t$). Note that 
\begin{equation}
\bE_x[\langle F_t(x), A \rangle ] = \sum_{(i_0,\ldots,i_r) \in \sT_k} (-1)^r \binom{k}{i_0 \ \cdots \ i_r} \langle M_{i_0} \otimes \cdots \otimes M_{i_r}, A \rangle. \label{eq:expression-derivative}
\end{equation}
Now we create an involution within the terms of $\sT_k$ that matches terms of equal magnitudes but opposite signs. This will establish that the RHS of  \eqref{eq:expression-derivative} vanishes. 

For every term $(i_0,\ldots,i_r)$ with $i_0 \neq 0$, match it with the 
term $(0,i_0,\ldots,i_r)$. Conversely, if $i_0 = 0$, match it with $(i_1,\ldots,i_r)$. 
This is an involution which preserves $\binom{k}{i_0 \ \cdots \ i_r}$ and 
$M_{i_0} \otimes \cdots \otimes M_{i_r}$ but negates $(-1)^r$, so all terms will cancel. 
The only exception is if $i_0 = 1$ (then $(0,i_0,\ldots,i_r) \not\in \sT_k$ so the 
involution fails), but this term is already zero because $M_1 = \bE[x-\mu] = 0$.
\end{proof}
With Lemma~\ref{lem:zero} in hand, we now know that the partial derivatives of 
$f_A(x) = \langle F_t(x), A \rangle$ are all mean-zero, and so we can apply 
the higher-order Poincar\'e inequality (Fact~\ref{thm:adamczak}). 

We therefore have $\bE[f_A(x)^2] \leq (t!)^2 C_t \sigma^{2t} \|A\|_F^2$ for the constant 
$C_t$ for which Fact~\ref{thm:adamczak} holds. 
Using Fact \ref{fact:spectral} as before, we obtain the following corollary asserting 
that $\bE[F_t(x)^{\otimes 2}]$ has a sum-of-squares upper bound:
\begin{corollary}
\label{cor:Ft-cert}
For all $t$, we have
\[
\langle \bE[F_t(x)^{\otimes 2}], v^{\otimes 2t} \rangle = \bE[\langle F_t(x), v^{\otimes t} \rangle^2] \preceq_{2t} (t!)^2 C_t \sigma^{2t} \|v\|_2^{2t}.
\]
\end{corollary}

\paragraph{Proving Theorem~\ref{thm:main-cert}}
We are now ready to prove Theorem~\ref{thm:main-cert}. We will first show that 
Theorem~\ref{thm:main-cert} holds in the infinite-data limit, and attend to finite-sample 
concentration at the end.

As in the case of $t = 3$, the strategy is to write 
$\langle M_{2t}, v^{\otimes 2t} \rangle$ as a combination of 
$\langle \bE[F_t(x)^{\otimes 2}], v^{\otimes 2t} \rangle$ together with terms that we can 
upper bound by recursively relying on estimates from smaller values of $t$. 
To aid in this, we use the following generalization of H\"older's inequality: 
\begin{lemma}
\label{lem:coalesce}
For every collection of non-negative integers $i_0, \ldots, i_r$ that sum to $2t$ and are each at most 
$2t-2$, $\langle M_{i_0} \otimes \cdots \otimes M_{i_r}, v^{\otimes 2t} \rangle \preceq_{2t} \langle M_2 \otimes M_{2t-2}, v^{\otimes 2t} \rangle$.
\end{lemma}
We will simply use Lemma~\ref{lem:coalesce} to replace all lower-order terms with 
terms of the form $\langle M_2 \otimes M_{2t-2}, v^{\otimes 2t} \rangle$. The proof of 
Lemma~\ref{lem:coalesce} involves providing sum-of-squares proofs for a number of standard 
polynomial inequalities, and is deferred to Appendix~\ref{sec:coalesce-proof}.

We will next induct on $t$. The base case $t = 1$ was already given above. For $t > 1$, 
we have 
\begin{align}
\langle M_{2t}, v^{\otimes 2t} \rangle &= \bE[\langle (x-\mu)^{\otimes t}, v^{\otimes t} \rangle^2] \\
 &\stackrel{(i)}{=} \bE[\langle F_t(x), v^{\otimes t} \rangle^2] -2\bE[\langle (x-\mu)^{\otimes t}, v^{\otimes t}\rangle\langle F_t(x) - (x-\mu)^{\otimes t}, v^{\otimes t}\rangle] - \bE[\langle F_t(x) - (x-\mu)^{\otimes t}, v^{\otimes t} \rangle^2] \\
 &\stackrel{(ii)}{\preceq_{2t}} (t!)^2C_t\sigma^{2t} \|v\|_2^{2t} - 2\bE[\langle (x-\mu)^{\otimes t}, v^{\otimes t} \rangle\langle F_t(x) - (x-\mu)^{\otimes t}, v^{\otimes t} \rangle].
\label{eq:intermediate-bound}
\end{align}
Here (i) is direct algebra while (ii) is applying Corollary~\ref{cor:Ft-cert}, as well as the 
fact that $-\bE[\langle F_t(x) - (x-\mu)^{\otimes t}, v^{\otimes t}\rangle^2] \preceq_{2t} 0$.

We next want to bound the $-2\bE[\langle (x-\mu)^{\otimes t}, v^{\otimes t}\rangle \langle F_t(x) - (x-\mu)^{\otimes t}, v^{\otimes t} \rangle]$ term. Note that $F_t(x) - (x-\mu)^{\otimes t}$ 
is simply $F_t(x)$ without its leading term. Therefore, recalling the definition 
\eqref{eq:Ft-def} of $F_t(x)$, we will let $\sT_t' =  \sT_t \backslash \{(t)\}$ (since 
the tuple $(t)$ is the one generating the $(x-\mu)^{\otimes t}$ term in $F_t$). Then we have
\begin{align}
\lefteqn{-2\bE[\langle (x-\mu)^{\otimes t}, v^{\otimes t} \rangle \langle (F_t(x) - (x-\mu)^{\otimes t}), v^{\otimes t} \rangle]} \\
 &= -2\Big\langle \sum_{(i_0,\ldots,i_r) \in \sT_t'} (-1)^r \binom{t}{i_0 \ \cdots \ i_r} M_{i_0+t} \otimes  M_{i_1} \otimes \cdots \otimes M_{i_r}, v^{\otimes 2k} \Big\rangle \\
 &\stackrel{(i)}{\preceq_{2t}} 2\Big\langle \sum_{(i_0,\ldots,i_r) \in \sT_t'} \binom{t}{i_0 \ \cdots \ i_r} M_{2t-2} \otimes M_2, v^{\otimes 2t} \Big\rangle \\
 &= 2\Big(\sum_{(i_0,\ldots,i_r) \in \sT_t'} \binom{t}{i_0 \ \cdots \ i_r}\Big) \langle M_{2t-2} \otimes M_2, v^{\otimes 2t} \rangle.
\label{eq:lower-order-bound}
\end{align}
Here (i) is by Lemma~\ref{lem:coalesce}.
We will bound the sum with a combinatorial argument. We can interpret the sum as the number of 
ways of splitting $\{1,\ldots,t\}$ into some number $r$ of sets such that all sets but the 
first have size at least $2$. Since $r \leq t/2+1$, this is bounded above by the number of 
ways of splitting $\{1,\ldots,t\}$ into $t/2+1$ (possibly empty) sets. But this is just 
$(t/2+1)^t$, as each element can freely go into one of the $t/2+1$ sets. Therefore, 
\eqref{eq:lower-order-bound} is at most $2(t/2+1)^t \langle M_{2t-2} \otimes M_2, v^{\otimes 2t} \rangle$. Plugging back into \eqref{eq:intermediate-bound}, we get
\begin{equation}
\langle M_{2t}, v^{\otimes 2t} \rangle \preceq_{2t} (t!)^2C_t \sigma^{2t} \|v\|_2^{2t} + 2(t/2+1)^t \langle M_{2t-2} \otimes M_2, v^{\otimes 2t} \rangle.
\end{equation}
Now by the inductive hypothesis, we have both 
$\langle M_2, v^{\otimes t} \rangle \preceq_2 C_2' \sigma^2 \|v\|_2^2$ and 
$\langle M_{2t-2}, v^{\otimes (2t-2)} \rangle \preceq_{2t-2} C_{2t-2}' \sigma^{2t-2} \|v\|_2^{2t-2}$ for some constants $C_2'$ and $C_{2t-2}'$. Therefore, 
$\langle M_{2t-2} \otimes M_2, v^{\otimes 2t} \rangle = \langle M_{2t-2}, v^{\otimes (2t-2)} \rangle \langle M_2, v^{\otimes 2} \rangle \preceq_{2t} C_2'C_{2t-2}' \sigma^{2t} \|v\|_2^{2t}$ 
(by Fact~\ref{fact:mult}).
We therefore see that $\langle M_{2t}, v^{\otimes 2t} \rangle \preceq_{2t} C_{2t}' \sigma^{2t} \|v\|_2^{2t}$, where we can take $C_{2t}' = (t!)^2 C_{t} + 2(t/2+1)^t C_2'C_{2t-2}'$.
This completes the induction.

\paragraph{Finite-sample concentration}
To finish the proof of Theorem~\ref{thm:cert}, it remains to establish finite-sample 
concentration. The key observation is that, as long as Corollary~\ref{cor:Ft-cert} 
holds for $s = 1, \ldots, t$, then all of the above steps go through. But 
Corollary~\ref{cor:Ft-cert} relies only on a bound on the maximum eigenvalue 
of the covariance of $F_t(x)$ (when $F_t(x)$ is flattened to a vector). 
It therefore suffices to bound this maximum eigenvalue 
in finite samples.

It is tempting to apply standard matrix concentration bounds (such as 
the matrix Chernoff inequality; see Theorem 5.5.1 of \citet{tropp2015introduction}); 
alas, $F_t(x)$ is too heavy-tailed for this to be valid (i.e., it is 
unbounded and does not even have exponential moments as required for 
most common matrix concentration bounds). We must instead appeal to the 
Matrix Rosenthal inequality (Corollary 7.4 of \citep{mackey2014matrix}). In our 
context, if we let $\vc(F_t(x))$ denote the flattening of $F_t(x)$ to a vector 
and $\Sigma = \bE[\vc(F_t(x))\vc(F_t(x))^{\top}]$ denote the covariance of 
this vector (and $\hat{\Sigma}$ denote the empirical covariance given $n$ samples), 
then the Matrix Rosenthal inequality states that for any $p \geq 1.5$, we have
\begin{align}
\bE[\tr(\hat{\Sigma}^{2p})]^{1/2p} 
 &\leq \big(\sqrt{\tr(\Sigma^{2p})^{1/2p}} + \sqrt{\frac{4p-2}{n}} n^{1/4p} \bE[\|F_t(x)\|_F^{4p}]^{1/4p}\big)^2 \\
 &\leq 2\big(\tr(\Sigma^{2p})^{1/2p} + \frac{4p-2}{n^{1-1/2p}} \bE[\|F_t(x)\|_F^{4p}]^{1/2p}\big).
\end{align}
Using a stronger version of Adamczak and Wolff's result 
(given as Theorem 3.3 of \citep{adamczak2015concentration}) 
we obtain the bound $\bE[\|F_t(x)\|_F^{4p}] \leq (C_t d^{t/2} p^t \sigma^t)^{4p}$ for some 
constant $C_t$. We also have 
$\tr(\Sigma^{2p})^{1/2p} \leq d^{t/2p} \|\Sigma\|_{\op} \leq C_t d^{t/2p} \sigma^{2t}$ 
as $\|\Sigma\|_{\op}$ is bounded by Corollary~\ref{cor:Ft-cert} 
(note that here and below the constant $C_t$ changes in each instance). 
Plugging into the inequality above, we obtain
\begin{align}
\bE[\tr(\hat{\Sigma}^{2p})]^{1/2p} \leq C_t \sigma^{2t} \big( d^{t/2p} + d^t p^{2t+1} n^{1/2p-1}\big).
\end{align}
We will take $p^* = t\log(d/\delta) + \frac{1}{2}$ which yields
\begin{equation}
\bE[\tr(\hat{\Sigma}^{2p^*})]^{1/2p^*} \leq C_t \sigma^{2t} \big(1 + d^t \log^t(d/\delta) n^{1/2p^*-1} \big).
\end{equation}
This is bounded so long as $n \geq \big((d\log(d/\delta))^{\frac{2p^*}{2p^*-1}}\big)^t$. 
For the value of $p^*$ above one can check that $(d\log(d/\delta))^{\frac{2p^*}{2p^*-1}} \leq 2d\log(d/\delta)$, so the above is bounded by $C_t \sigma^{2t}$ so long as 
$n \geq (2d\log(d/\delta))^t$.

We thus have a bound on $\bE[\tr(\hat{\Sigma}^{2p^*})]$ which we would like to turn 
into a high-probability bound on $\|\hat{\Sigma}\|_{\op}$. For this we make use of 
the matrix Chebyshev bound (Proposition 6.2 of \citep{mackey2014matrix}), which in our 
case implies that 
$\bP[\|\hat{\Sigma}\|_{\op} \geq 2\bE[\tr(\hat{\Sigma}^{2p^*})]^{1/2p^*}] \leq 2^{-p^*} \leq \delta$. Union bounding over the $t$ instances where we must invoke 
Corollary~\ref{cor:Ft-cert} completes the proof.


\section{Robust Clustering}
\label{sec:clustering}




In this section, we give our algorithms for robust clustering under separation assumptions and 
for robust mean estimation. 
The first main result of this section is that one can efficiently recover a good clustering 
of data, even in the presence of outliers, for arbitrary data that is $(2t,B)$-SOS-certifiable 
(i.e., where the $2t$th moments are bounded for all pseudodistributions on the sphere).

\begin{theorem}
\label{thm:clustering-separated}
Suppose $x_1, x_2, \ldots, x_n \in \R^d$ satisfies the following clusterability condition:
It can be partitioned into sets $I_1, \ldots, I_k$, where 
$I_j$ has size $\alpha_j n$, together with an $\epsilon$ fraction of arbitrary outliers,  
where $\epsilon = 1-(\alpha_1 + \cdots + \alpha_k)$. 
Furthermore, $\epsilon \leq \frac{\alpha}{8}$, where $\alpha = \min_j \alpha_j$, and 
the sets $I_j$ satisfy:
\begin{enumerate}

  \item For every $j$, $I_j$ is SOS-certifiable around its mean $\mu_j = \E_{x \sim I_j}[x]$: 
  \begin{equation}
  \|\E_{x \sim I_j}[(x-\mu_j)^{\otimes 2t}]\|_{\sos_{2t}} \le B.\label{eq:bounded-moment-condition}
  \end{equation}
  \item For every $i,j$, the means of distinct clusters are well-separated:
  \begin{equation}
  \label{eq:separation-condition}
  \|\mu_i - \mu_j\| \geq \Csep B \alpha^{-1/t},
  \end{equation}   
  with $\Csep \geq C_0$ (with $C_0$ an absolute constant).
\end{enumerate}

Then, 
there is a $(nd)^{O(t)}$ time algorithm 
that outputs $\hat{\mu}_1, \ldots, \hat{\mu}_k$ satisfying for each $j=1,\ldots,k$,
\[
\|\hat{\mu}_j - \mu_j\|_2 \leq  B \cdot \oo\Paren{\frac{\epsilon}{\alpha} + \Csep^{-2t}}^{\frac{2t-1}{2t}}.
\] 

\end{theorem}

\begin{remark}
Note that if $\epsilon \geq \alpha$, then the number of outliers can be as large as the 
smallest cluster. In that case, it is information theoretically impossible, in general, 
to find a unique list of $k$ correct means. Whenever $\epsilon \leq \alpha/8$, 
Theorem~\ref{thm:clustering-separated} gives a non-trivial guarantee on the recovered means.
\end{remark}
\begin{remark}
As explained in Section~\ref{sec:intro}, the dependence of the error on the separation $\Csep$ is 
also information-theoretically necessary, even in one dimension. In particular, 
consider two clusters drawn from distributions with slightly overlapping tails; then 
it is impossible to tell whether a point in the overlap should come from one cluster 
or the other, which will lead to small but non-zero errors in the estimated means.
\end{remark}

The certified bounded moment condition \eqref{eq:bounded-moment-condition} is satisfied 
by data generated from a large class of general mixture models. 
As a corollary of Theorem \ref{thm:clustering-separated}, 
we obtain results for learning means in general mixture models (see 
Corollary~\ref{cor:poincare-mixture}). 

Our second result is outlier-robust mean estimation where an $\epsilon$ fraction of 
the input points are arbitrary outliers.

\begin{theorem}
\label{thm:robust-clustering-eps}
Let $x_1, \ldots, x_n \in \R^d$ be such that there exists an unknown subset 
$I \subseteq [n]$ of size $(1-\epsilon) n$ satisfying 
$\|\E_{x \sim I}[(x_i-\mu)^{\otimes 2t}]\|_{\sos_{2t}} \leq B$,
where $\mu = \bE_{x \sim I}[x]$. 
If $\epsilon \leq \frac{1}{4}$, 
then there is an algorithm that runs in time $(nd)^{\oo(t)}$ and outputs 
an estimate $\hat{\mu}$ such that $\|\hat{\mu} - \mu\|_2 \leq \oo(B \epsilon^{1-1/2t})$.
\end{theorem}
Theorem~\ref{thm:robust-clustering-eps} is in fact a corollary of 
Theorem~\ref{thm:clustering-separated} (in the special case of a single cluster), 
but we state it separately for emphasis (and because Theorem~\ref{thm:clustering-separated} 
actually requires Theorem~\ref{thm:robust-clustering-eps} in its proof). 
The error $\epsilon^{1-1/2t}$ interpolates between existing results 
which achieve error $\sqrt{\epsilon}$ for distributions with bounded covariance, 
and those achieving $\tilde{\oo}(\epsilon)$ for e.g. Gaussian distributions. 

Finally, we can obtain results in an even more extreme setting, 
where all but an $\alpha$ fraction of the input points are outliers, for $\alpha$ 
potentially smaller than $\frac{1}{2}$. 
This corresponds to the robust learning setting proposed in \citet{charikar2017learning}, 
and our algorithm works in the \emph{list-decodable learning} model 
\citep{balcan2008discriminative} in which a short 
list of $\oo(1/\alpha)$ candidate means is allowed to be output. (Note that this 
is information-theoretically necessary if $\alpha < \frac{1}{2}$.)
\begin{theorem}
\label{thm:robust-clustering-alpha}
Let $x_1, \ldots, x_n \in \R^d$ be such that there exists a subset $I \subseteq [n]$ of size 
$\alpha n$ that satisfies $\|\E_{x \sim I} (x_i-\mu)^{\otimes 2t}\|_{\sos_{2t}} \leq B$ for some $\mu \in \R^d$. 
Then, there is an algorithm that runs in time $(nd)^{\oo(t)}$ and 
outputs estimates $\hat{\mu}_1, \ldots, \hat{\mu}_h$ with $h \leq \frac{4}{\goodfrac}$ 
and $\min_{j=1}^h \|\hat{\mu}_j - \mu\|_2 \leq \oo(B / \alpha^{1/t})$. 
\end{theorem}
Theorem~\ref{thm:robust-clustering-alpha} implies results for robust clustering, 
as we can think of each component $I_j$ of the cluster as being the set $I$, 
and Theorem~\ref{thm:robust-clustering-alpha} 
then says we will output a list of $4/\alpha$ candidate 
means such that the mean of every cluster $I_j$ is within $\oo(B / \alpha^{1/t})$ of a candidate 
mean $\hat{\mu}_{j'}$. 
This is weaker than the guarantee of Theorem~\ref{thm:clustering-separated}, 
but holds even when the clusters are not well-separated.

In fact, while Theorem~\ref{thm:clustering-separated} 
may appear much stronger in the well-separated setting, 
it will follow as a basic extension of 
the ideas in Theorems~\ref{thm:robust-clustering-alpha} and \ref{thm:robust-clustering-eps}. 
Most of the rest of this section will be devoted to proving these two theorems, 
with Theorem~\ref{thm:clustering-separated} handled at the end.





\subsection{Basic Clustering Relaxation}
Key to our algorithmic results is a natural hierarchy of convex relaxations for recovering an 
underlying clustering. In what follows, we will refer to this as the basic clustering relaxation.

The idea is as follows. Assume (as in our results above) that there is some set $I$ 
such that $\|\bE_{x \sim I}[(x-\mu)^{\otimes 2t}]\|_{\sos_{2t}} \leq B$. 
This is the same as saying that for all pseudodistributions 
$\xi(v)$ over the unit sphere, 
$\frac{1}{|I|} \sum_{x \in I} \tilde{\bE}_{\xi(v)}[\langle x-\mu, v \rangle^{2t}] \leq B^{2t}$.
Motivated by this, we might seek $w_1, \ldots, w_n$ such that 
\begin{equation}
\label{eq:just-xi}
\frac{1}{n} \sum_{i=1}^n \tilde{\bE}_{\xi(v)}[\langle x_i - w_i, v \rangle^{2t}] \text{ is small for all pseudodistributions } \xi \text{ over the unit sphere.}
\end{equation}
The reason we allow distinct $w_i$ is because in general the 
data might consist of multiple clusters and so we want to allow flexibility for the $w_i$ to 
fit more than a single cluster at once.

However, as stated, the trivial solution $w_i = x_i$ will always have zero cost. Intuitively, 
the reason is that the $w_i$ are free to completely overfit the $x_i$. We would like to impose 
an additional penalty term to keep $w_i$ from overfitting $x_i$ too much. In fact, the 
right metric of overfitting turns out to be 
\begin{equation}
\label{eq:overfit}
\frac{1}{|I|} \sum_{i \in I} \langle w_i - \mu, w_i \rangle^{2t},
\end{equation}
as we will see below. We cannot directly control $w_i-\mu$ (as we do not know $\mu$), 
but we can observe that $w_i - \mu = (w_i - x_i) + (x_i - \mu)$, and that 
both $w_i-x_i$ and $x_i-\mu$ have small moments with respect to any pseudodistribution 
$\xi$ (by the constraint \eqref{eq:just-xi}, and by our assumption on $\mu$). 
So, we could conservatively try to control $\langle z_i, w_i \rangle^{2t}$ for 
\emph{all possible} sets of points $z_i$ that have small moments (which would 
in particular control $w_i - x_i$ and $x_i - \mu$).

This motivates us to add the following additional constraint: for any $z_1, \ldots, z_n$ 
such that $\sum_{i=1}^n \tilde{\bE}_{\xi'(v)}[\langle z_i, v \rangle^{2t}] \leq 1$ 
for all pseudodistributions $\xi'$ on the sphere, we ask that 
$\sum_{i=1}^n \langle z_i, w_i \rangle^{2t}$ be small (again, we think of the $z_i$ as standing 
in for the points $w_i - x_i$ and $x_i - \mu$).
This constraint on the $w_i$ is not convex, but we can take a sum-of-squares relaxation 
by replacing the $z_i$ with pseudodistributions $\zeta_i(z_i)$. This turns out to ask 
that 
\begin{equation}
\label{eq:just-zeta}
\sum_{i=1}^n \tilde{\bE}_{\zeta_i(z_i)}[\langle z_i, w_i \rangle^{2t}] \text{ is small whenever } \sum_{i=1}^n \tilde{\bE}_{\zeta_i(z_i)}[z_i^{\otimes 2t}] \psos \sI,
\end{equation}
where $\sI$ is the order-$2t$ identity tensor and $T_1 \psos T_2$ means that 
$\langle T_1, v^{\otimes 2t} \rangle \psos \langle T_2, v^{\otimes 2t} \rangle$ 
(as a polynomial in $v$).

The basic clustering relaxation, defined in 
Algorithm~\ref{alg:basic-relaxation} below, asks to either find $w_1, \ldots, w_n$ 
such that both \eqref{eq:just-xi} and \eqref{eq:just-zeta} are small, or else 
to find dual certificates $\xi$, $\zeta$ proving that they cannot be small.
\newcommand{\thresh}{\Gamma}
\begin{algorithm}
\caption{Basic Clustering Relaxation}
\label{alg:basic-relaxation}
\setlength{\abovedisplayskip}{0pt}
\setlength{\belowdisplayskip}{0pt}
\setlength{\abovedisplayshortskip}{0pt}
\setlength{\belowdisplayshortskip}{0pt}
\begin{algorithmic}[1]
\STATE Input: $X = x_1,\ldots,x_n \in \R^d$, weights $c_1, \ldots, c_n \in [0,1]$, multiplier $\lambda$, threshold $\thresh$.
\STATE Find either $w_1, \ldots, w_n$ such that 
\begin{equation}
\sum_{i=1}^n c_i (\pE_{\xi(v)}[\iprod{ x_i - w_i, v }^{2t}] + \lambda \pE_{\zeta_i(z_i)}[\iprod{ w_i, z_i }^{2t}]) \leq 2\thresh \label{eq:optimization} 
\end{equation}
for all pseudodistributions $\xi$ over the unit sphere and $\zeta_1, \ldots, \zeta_n$ with 
$\sum_{i=1}^n \pE_{\zeta_i(z_i)}[z_i^{\otimes 2t}] \psos \sI$, or else find 
$\xi$, $\zeta_{1:n}$ such that the expression in \eqref{eq:optimization} is at least 
$\thresh$
for all $w_{1:n}$.
\end{algorithmic}
\end{algorithm}

Note that Algorithm~\ref{alg:basic-relaxation} is basically asking to implement an 
approximate separation oracle for the expression in \eqref{eq:optimization}. 
The weights $c_1, \ldots, c_n$ will be used later to downweight outliers in the data.

It is easy to show that Algorithm~\ref{alg:basic-relaxation} specifies a convex primal-dual 
problem and can be solved in polynomial time. 
We give an argument in Section~\ref{sec:basic-relaxation-proof} of the Appendix for completeness:

\begin{lemma}[Solving Basic Clustering Relaxation]
\label{lem:basic-relaxation}
There is a polynomial time algorithm that solves the Basic Clustering Relaxation in time 
$(nd)^{\oo(t)}$.
\end{lemma}

The full algorithm uses the relaxation \eqref{eq:optimization} but must handle two 
additional issues. 
The first is \emph{outliers}, which can prevent \eqref{eq:optimization} from 
being small, and must be removed in a separate step. 
The second is the need for \emph{re-clustering}. This second issue arises 
because \eqref{eq:optimization} is not translation-invariant, and in particular the recovery 
error after running \eqref{eq:optimization} will depend on the $\ell_2$-norm $r = \|\mu\|_2$ of 
$\mu$. We can obtain improved bounds 
by clustering the $w_i$ output by \eqref{eq:optimization}, and then re-running the algorithm 
on each cluster with a smaller value of $r$. This is similar to the re-clustering idea 
in \citet{charikar2017learning}, but we obtain a much simpler proof by making use of 
the recent idea of \emph{resilience} introduced in \citet{steinhardt2018resilience}.

We analyze and give pseudocode for the outlier removal and re-clustering steps 
in Sections~\ref{sec:outlier} and \ref{sec:re-cluster} below. The output of the 
outlier removal algorithm already satisfies a basic but coarse error bound, 
given as Proposition~\ref{prop:clustering-bound}.

\subsection{Outlier Removal: Basic Bound}
\label{sec:outlier} 

The outlier removal algorithm is given in Algorithm~\ref{alg:clustering}. 
We maintain weights $c_i$ on the points $x_i$, and run Algorithm~\ref{alg:basic-relaxation} 
to attempt to find points $w_{1:n}^{\star}$. If we fail, then we downweight $c_i$ 
according to the value of 
$\tau_i^{\star} = \min_{w} \tau_i(w)$, where $\tau_i(w) = \pE[\langle x_i - w, v \rangle^{2t}] + \lambda \pE[\langle w, z_i \rangle^{2t}]$ 
(line~\ref{line:reweight}). This is intuitive because, whenever 
Algorithm~\ref{alg:basic-relaxation} fails to output $w_{1:n}$, it instead 
outputs a dual certificate $\xi$, $\zeta$ such that $\sum_{i=1}^n c_i \tau_i^{\star}$ is large. 
Thus, intuitively, points $i$ with a large value of $\tau_i^{\star}$ are responsible for 
\eqref{eq:optimization} being large, and should be downweighted.

\begin{algorithm}
\caption{Outlier Removal Algorithm}
\label{alg:clustering}
\label{alg:cluster}
\setlength{\abovedisplayskip}{0pt}
\setlength{\belowdisplayskip}{0pt}
\setlength{\abovedisplayshortskip}{0pt}
\setlength{\belowdisplayshortskip}{0pt}
\begin{algorithmic}[1]
\STATE Input: $x_1,\ldots,x_n$, $B$, $\alpha$, and upper bound $r$ on $\|\mu\|_2$.
\STATE Initialize $c \gets [1; \cdots; 1] \in \bR^n$
\STATE Set $\lambda \gets \alpha n (B/r)^{2t}$
\WHILE{{\bfseries true}}
  \STATE Run Algorithm~\ref{alg:basic-relaxation} with threshold $\thresh = 4(nB^{2t} + \lambda r^{2t}/\alpha)$ 
 to obtain either $w_{1:n}^{\star}$ or $\xi^{\star}$, $\zeta_{1:n}^{\star}$.
  \IF{$w_{1:n}^{\star}$ are obtained}
    \RETURN $w_{1:n}^{\star}$, $c_{1:n}$ 
  \ELSE
    \STATE Let $\tau_i^{\star} \gets \min_{w} \tau_i(w)$, where $\tau_i(w) = \pE_{\xi^{\star}(v)}[\langle x_i - w, v \rangle^{2t}] + \lambda \pE_{\zeta_i^{\star}(z_i)}[\langle w, z_i \rangle^{2t}]$.
    \STATE $c_i \gets c_i (1 - \tau_i^{\star} / \tau_{\max})$ for all $i$, where $\tau_{\max} = \max_{i=1}^n \tau_i^{\star}$.
\protect\label{line:reweight}
  \ENDIF
\ENDWHILE
\end{algorithmic}
\end{algorithm}

We analyze Algorithm~\ref{alg:clustering} in two steps. 
First, we show that if the value of \eqref{eq:optimization} is large (and hence 
$\xi^{\star}$, $\zeta_{1:n}^{\star}$ are obtained), then 
the re-weighting step (line~\ref{line:reweight}) downweights bad points much more 
than it downweights good points. 
Second, we show that if the value of \eqref{eq:optimization} 
is small, then the returned $w_i^{\star}$ constitute 
a good clustering (such that one of the clusters is centered close to the 
true mean $\mu$). 

Formally, we will show:
\begin{proposition}
\label{prop:clustering-bound}
Suppose that there is a set $\goodset \subseteq [n]$ of size $\alpha n$ such that 
$\|\frac{1}{|\goodset|} \sum_{i \in \goodset} (x_i-\mu)^{\otimes 2t}\|_{\sos_{2t}} \leq B$, 
and $\|\mu\|_2 \leq r$. Then the output $w_{1:n}$, $c_{1:n}$ 
of Algorithm~\ref{alg:clustering} satisfies the following property:
\begin{equation}
\label{eq:clustering-bound}
\frac{1}{|\goodset|} \sum_{i \in \goodset} c_i \|w_i^{\star} - \mu\|_2^{4t} \leq \oo((4 B r)^{2t} / \alpha^2)
\text{ and } \frac{1}{|I|} \sum_{i \in \goodset} (1-c_i) \leq \frac{1-\alpha}{3}.
\end{equation}
In particular, at least a $\frac{1}{2}$ fraction of 
the $i \in I$ satisfy $c_i \geq \frac{1}{4}$ 
and $\|w_i^{\star} - \mu\|_2^2 \leq \oo(r B \goodfrac^{-1/t})$.
Moreover, if $\alpha = 1-\epsilon$, with $\epsilon \leq \frac{1}{2}$, then at least 
$1-\epsilon$ of the $i \in I$ satisfy $c_i \geq \frac{1}{4}$ and 
$\|w_i^{\star} - \mu\|_2^2 \leq \oo(r B \epsilon^{-1/2t})$.
\end{proposition}
\begin{proof}
Given \eqref{eq:clustering-bound}, the remainder of Proposition~\ref{prop:clustering-bound} 
follows from repeated application of Markov's inequality. So, we will focus on establishing 
\eqref{eq:clustering-bound}, splitting into cases based on the value of 
\eqref{eq:optimization}. In particular, for the chosen threshold 
$\thresh = 4(nB^{2t} + \lambda r^{2t}/\goodfrac)$, note that either 
Algorithm~\ref{alg:basic-relaxation} outputs $\xi^{\star}$, $\zeta_{1:n}^{\star}$ showing 
that \eqref{eq:optimization} is large (at least $\thresh$), or it outputs $w_{1:n}^{\star}$ 
such that \eqref{eq:optimization} is small (at most $2\thresh$).

\textbf{Case 1: \eqref{eq:optimization} is large.}
Suppose that \eqref{eq:optimization} is large, 
and hence in particular 
$\sum_{i=1}^n c_i\tau_i^{\star} \geq 4(nB^{2t} + \lambda r^{2t}/\goodfrac)$. 
We start by showing that $\sum_{i \in I} \tau_i^{\star}$ is much smaller than this. 
Indeed, we have 
$\sum_{i \in I} \tau_i^{\star} \leq \sum_{i \in I} \tau_i(\mu)$.
To bound $\tau_i(\mu)$, first note that 
$\sum_{i \in I} \pE_{\xi(v)}[\langle x_i - \mu, v \rangle^{2t}] \leq |I|B^{2t}$, 
since we are assuming a sum-of-squares certificate for $I$. 
Also, 
$\sum_{i \in I} \pE_{\zeta_i(z_i)}[\langle \mu, z_i \rangle^{2t}] \leq \langle \mu^{\otimes 2t}, \sI \rangle \leq r^{2t}$. Therefore, 
\begin{equation}
\sum_{i \in I} c_i\tau_i^{\star} \leq \sum_{i \in I} \tau_i(\mu) \leq |I|B^{2t} + \lambda r^{2t}.
\end{equation}

Now, since $\sum_{i=1}^n c_i\tau_i^{\star} \geq 4(nB^{2t} + \lambda r^{2t}/\goodfrac)$ by assumption, 
the average of $\tau_i^{\star}$ over $\{1,\ldots,n\}$ is more than $4$ times 
larger than the average of $\tau_i^{\star}$ over $\goodset$ 
(since $|\goodset| \geq \goodfrac n$). Let $c_i$ and $c_i'$ be the values of 
the weights before and after the update on line~\ref{line:reweight}. 
We have, for any set $S$, 
\begin{equation}
\sum_{i \in S} c_i - c_i' = \frac{1}{\tau_{\max}} \sum_{i \in S} c_i \tau_i^{\star}.
\end{equation}
Therefore, the amount that the weights in a set $S$ decrease is proportional to
$\sum_{i \in S} c_i\tau_i^{\star}$. 
Since the average over $I$ is at most $\frac{1}{4}$ the average over $\{1,\ldots,n\}$, 
this means that the weights in $I$ decrease at most $\frac{1}{4}$ as fast as the 
weights overall. In particular, at the end of the algorithm we have  
$\sum_{i\in I} (1-c_i) \leq \frac{|I|}{4n} \big(\sum_{i \in I} (1-c_i) + \sum_{i \not\in I} (1-c_i)\big)$. Re-arranging yields $\frac{1}{|I|} \sum_{i \in I} (1-c_i) \leq \frac{1-|I|/n}{4-|I|/n} \leq \frac{1-|I|/n}{3}$. In particular, at most $\frac{1-\alpha}{3}$ of the weight of the  
good $c_i$ is removed at any stage in the algorithm (and hence in particular we eventually 
end up in case $2$).

\textbf{Case 2: \eqref{eq:optimization} is small.} 
Note that $\lambda$ is chosen so that 
$8(nB^{2t} + \lambda r^{2t}/\alpha) = 16nB^{2t}$. Therefore, 
if \eqref{eq:optimization} is small, then 
$\sum_{i \in I} c_i \pE_{\xi(v)}[\langle x_i - w_i^{\star}, v \rangle^{2t}] \leq 16nB^{2t}$ 
for all pseudodistributions $\xi(v)$ over the unit sphere, 
and by assumption $\sum_{i \in I} c_i \pE_{\xi(v)}[\langle x_i - \mu, v \rangle^{2t}] \leq |I|B^{2t} \leq nB^{2t}$ as well. Combining these, we obtain 
$\sum_{i \in I} c_i \bE_{\tilde{\mu}(v)}[\langle w_i^{\star} - \mu, v \rangle^{2t}] \leq 2^{2t-1} \cdot 17nB^{2t}$. So, $w_i - \mu$ has small $2t$th pseudomoment.

In addition, since \eqref{eq:optimization} is small we know that 
$\sum_{i \in I} c_i \pE_{\zeta_i(z_i)}[\langle w_i^{\star}, z_i \rangle^{2t}] \leq 16r^{2t}/\alpha$ 
whenever the $z_i$ have small $2t$th pseudomoment. 
Putting these together, we get
\begin{align}
\sum_{i \in I} \|w_i^{\star} - \mu\|_2^{4t}
 &= \sum_{i \in I} \langle w_i^{\star} - \mu, w_i^{\star} - \mu \rangle^{2t} \\
 &\leq 2^{2t-1} \sum_{i \in I} \langle w_i^{\star} - \mu, \mu \rangle^{2t} + \langle w_i^{\star} - \mu, w_i^{\star} \rangle^{2t} \\
 &\leq 2^{2t-1} \Big((2^{2t-1} 17nB^{2t}) \cdot \|\mu\|_2^{2t} + (2^{2t-1} 17nB^{2t}) \cdot (16r^{2t}/\alpha)\Big) \\
 &\leq 2^{4t-2} \cdot 289n (B r)^{2t} / \alpha.
\end{align}
In particular, we have 
\begin{equation}
\frac{1}{|I|} \sum_{i \in I} \|w_i^{\star} - \mu\|_2^{4t} \leq \oo(1) \cdot (4B r)^{2t} \cdot (n/\alpha |I|) = \oo((4 B r)^{2t} / \alpha^2).
\label{eq:clustering-bound-pre}
\end{equation}
This, together with the earlier bound $\frac{1}{|\goodset|} \sum_{i \in \goodset} (1-c_i) \leq \frac{1-|\goodset|/n}{3}$, yields \eqref{eq:clustering-bound}, which establishes the 
proposition.
\end{proof}

\subsection{Proving Theorem~\ref{thm:robust-clustering-eps} via Re-Centering}

Proposition~\ref{prop:clustering-bound} has a dependence on $r$, which we would like 
to get rid of. 
If $\alpha = 1-\epsilon$ where $\epsilon \leq \frac{1}{4}$, then there is a fairly 
simple strategy for doing so, based on approximately recovering the mean $\mu$ 
(with error that depends on $\sqrt{r}$) and re-running Algorithm~\ref{alg:clustering} 
after re-centering the points around the approximate mean.

The following key fact about the output of Algorithm~\ref{alg:clustering} 
will aid us in this (proof in Section~\ref{sec:resilience-proof}):
\newcommand{\rad}{\rho}
\begin{proposition}
\label{prop:resilience-small}
Let $w_{1:n}, c_{1:n}$ be the output of Algorithm~\ref{alg:clustering}. 
Let $S \subseteq \{1,\ldots,n\}$ be a subset of at least $\frac{n}{2}$ points such that (1) 
$c_i \geq \frac{1}{4}$ for all $i \in S$, and (ii) 
$\|w_i - \hat{\mu}\|_2 \leq \rad$ for all $i \in S$, where $\hat{\mu}$ is the mean of $S$. 

Then, for any subset $S'$ of $S$ with $|S'| \geq (1-\epsilon_1)|S|$ and 
$\epsilon_1 \leq \frac{1}{2}$, the mean of $S'$ is close to the mean $\hat{\mu}$ of $S$:
\begin{equation}
\Big\|\frac{1}{|S'|} \sum_{i \in S'} (x_i - \hat{\mu})\Big\|_2 \leq 2\epsilon_1 \rad + \oo(B \epsilon_1^{1-1/2t}).
\end{equation}
\end{proposition}
The property that all large subsets of $S$ have mean close to the mean of $S$ is called 
\emph{resilience}, and was introduced in \citet{steinhardt2018resilience}.
There it is shown that one can obtain a good approximation to the true mean 
$\mu$ by finding any large resilient set. Here we repeat the short argument of this fact, 
and show that a large resilient set can be found efficiently.

First note that when $\epsilon = 1-\alpha \leq \frac{1}{2}$, 
Proposition~\ref{prop:clustering-bound} guarantees that 
a large subset $S_0$ of at least $1-\epsilon$ of the points in $I$ satisfies the conditions of 
Proposition~\ref{prop:resilience-small} with radius 
$\rad = C \sqrt{rB \epsilon^{-1/2t}}$ for some appropriate constant $C$. 
Moreover, we can efficiently find such a set (with radius at most $4\rad$) as follows:
\begin{itemize}
\item Find any point $i_0$ such that there are at least $(1-\epsilon)n$ points $j$ 
      with $\|w_{i_0} - w_j\|_2 \leq 2\rad$ and $c_j \geq \frac{1}{4}$.
\item Output these $(1-\epsilon)n$ indices $j$ as the set $S$.
\end{itemize}
We know that some set $S$ will be found by this procedure (since any $i_0 \in S_0$ 
will satisfy the conditions). Moreover, by the triangle inequality all points in $S$ 
will be distance at most $4\rad$ from the mean of $S$.

Now, both $S$ and $S_0$ satisfy the assumptions of Proposition~\ref{prop:resilience-small}, 
and their intersection $S \cap S_0$ has size at least $(1-2\epsilon)n \geq (1-2\epsilon)|S|$. 
Therefore, assuming $\epsilon \leq \frac{1}{4}$, 
the mean $\mu$ of $S$ is close to the mean $\hat{\mu}$ of $S_0$ (since they are both 
close to the mean of $S \cap S_0$ by Proposition~\ref{prop:resilience-small}). This yields:
\begin{equation}
\label{eq:mu-mu-hat}
\|\mu - \hat{\mu}\|_2 = \oo(\epsilon \rad + B \epsilon^{1-1/2t}) = \oo(\epsilon \sqrt{rB \epsilon^{-1/2t}} + B\epsilon^{1-1/2t}).
\end{equation}
By finding any set $S_0$ as above and computing its mean, we obtain
a point $\hat{\mu}$ that satisfies \eqref{eq:mu-mu-hat}.
%
We can then re-run Algorithm~\ref{alg:clustering} centered at the new point $\hat{\mu}$ 
(i.e., shifting all of the points by $\hat{\mu}$ as a pre-processing step). When we 
do this, we can use a smaller radius $r' = C \cdot (\sqrt{r B \epsilon^{2-1/2t}} + B \epsilon^{1-1/2t})$ for an appropriate universal constant $C$. 
After repeating this a small number of times, we eventually 
end up with a radius that is $\oo(B \epsilon^{1-1/2t})$. This yields 
Theorem~\ref{thm:robust-clustering-eps}.

\subsection{Proving Theorem~\ref{thm:robust-clustering-alpha} via Re-Clustering}
\label{sec:re-cluster}

We now turn to the more complicated case where potentially $\alpha \ll 1$. 
Proposition~\ref{prop:clustering-bound} says that Algorithm~\ref{alg:clustering} somewhat 
succeeds at identifying the true mean: while initially we only knew that 
$\|\mu\|_2 \leq r$, the output of Algorithm~\ref{alg:clustering} has at least 
$\frac{\goodfrac n}{2}$ points such that 
$\|w_i - \mu\|_2 \leq \oo(\sqrt{r B \goodfrac^{-1/m}})$. However, 
our goal is to obtain a point $w$ such that $\|w - \mu\|_2 = \oo(B \goodfrac^{-1/m})$, with 
no dependence on $r$. To obtain this stronger bound, the intuition is to 
sub-divide the $w_i$ into clusters of radius 
$\oo(\sqrt{rB \goodfrac^{-1/m}})$, and then recursively apply 
Algorithm~\ref{alg:clustering} to each cluster.

As in the previous section, we will make use of the resilience property in order 
to find a good clustering. In this case, we need the following analog of 
Proposition~\ref{prop:resilience-small} (proof in Section~\ref{sec:resilience-proof}):
\begin{proposition}
\label{prop:resilience-large}
Let $w_1, \ldots, w_n$ be the output of Algorithm~\ref{alg:clustering}. 
Let $S \subseteq \{1,\ldots,n\}$ be a subset of points such that (1) 
$c_i \geq \frac{1}{4}$ for all $i \in S$, and (ii) 
$\|w_i - \hat{\mu}\|_2 \leq \rad$ for all $i \in S$ for some $\hat{\mu}$.

Then, for any subset $S'$ of $S$ with $|S'| \geq \beta n$, the 
mean of $S'$ is close to $\hat{\mu}$:
\begin{equation}
\Big\|\frac{1}{|S'|} \sum_{i \in S'} (x_i - \hat{\mu})\Big\|_2 \leq \rho + \oo(B / \beta^{1/2t}).
\end{equation}
\end{proposition}
Again, sets $S$ satisfying the conclusion of Proposition~\ref{prop:resilience-large} are 
called resilient sets.
Similarly to the case above where $\alpha = 1-\epsilon$, \citet{steinhardt2018resilience}
shows that, by covering $\{1,\ldots,n\}$ with approximately disjoint resilient 
sets, one can obtain a list of $\oo(1/\alpha)$ candidate means $\hat{\mu}$ such that one 
of the candidates is close to the true mean $\mu$. We will review their argument here, 
extending it to suit our purposes.

Our final algorithm is going to run Algorithm~\ref{alg:clustering} centered at a number 
of different points $\tilde{\mu}_1, \ldots, \tilde{\mu}_b$ and then consolidate their 
outputs into a small number of clusters (and then iteratively re-run 
Algorithm~\ref{alg:clustering} at a new set of points based on the centers of those clusters).

Specifically, suppose that Algorithm~\ref{alg:clustering} has just been run 
around points $\tilde{\mu}_1, \ldots, \tilde{\mu}_b$, and that some unknown 
$\tilde{\mu}_{a^*}$ (with $1 \leq a^* \leq b$) is close to the target $\mu$:
$\|\tilde{\mu}_{a^*} - \mu\|_2 \leq R$. Let 
$w_{1:n}^{(1)}, \ldots, w_{1:n}^{(b)}$ and $c_{1:n}^{(1)}, \ldots, c_{1:n}^{(b)}$ 
denote the outputs of Algorithm~\ref{alg:clustering} on $\tilde{\mu}_1, \ldots, \tilde{\mu}_b$. 
By Proposition~\ref{prop:clustering-bound}, we have that 
$c_i^{(a^*)} \geq \frac{1}{4}$ and $\|w_i^{(a^*)} - \mu\|_2 \leq \rad$ 
for at least $\frac{\goodfrac n}{2}$ values of $i$, 
with $\rad = C \sqrt{RB/\alpha^{1/t}}$.

Now, consider any maximal collection of sets 
$S_1, \ldots, S_m \subseteq [n]$ with the following properties:
\begin{itemize}
\item $|S_j| \geq \frac{\alpha}{4} n$ for all $j$.
\item $S_j$ and $S_{j'}$ are disjoint for $j \neq j'$
\item For each $j$, there is some $a$ and $i_0 \in S_j$ such that $\|w_i^a - w_{i_0}^a\|_2 \leq 2\rad$ for all $i \in S_j$, and moreover $c_i^a \geq \frac{1}{4}$ for all $i \in S_j$.
\end{itemize}
Note that such a collection can be found efficiently via a greedy algorithm: For each $i_0$ and 
$a$, check if there are at least $\frac{\alpha}{4}n$ indices $i$ which do not yet lie in any of 
the existing $S_j$, and such that $\|w_i^a - w_{i_0}^a\|_2 \leq 2\rad$ and $c_i^a \geq \frac{1}{4}$. If there are, we 
can add those indices $i$ as a new set $S_j$, and otherwise the collection of existing $S_j$ 
is maximal. Similarly to \citet{steinhardt2018resilience}, we then have:
\begin{proposition}
\label{prop:covering}
For each $S_j$, let $\hat{\mu}_j$ be the mean of $S_j$. Then, at least one of the 
$\hat{\mu}_j$ satisfy $\|\hat{\mu}_j - \mu\|_2 \leq C' (\sqrt{RB/\alpha^{1/t}} + B/\alpha^{1/t})$ 
for some universal constant $C'$.
\end{proposition}
\begin{proof}
Note that the set $S = \{i \mid c_i^{(a^*)} \geq \frac{1}{4} \text{ and } \|w_i^{(a^*)} - \mu\|_2 \leq \rho\}$ satisfies the necessary conditions to be a valid set $S_j$, and has 
size at least $\frac{\goodfrac n}{2}$. 
Therefore, by the maximality of the collection $S_1, \ldots, S_m$, at least 
$\frac{\alpha}{4}n$ points in $S$ must lie in $S_1 \cup \cdots \cup S_m$ (as otherwise 
$S \backslash (S_1 \cup \cdots \cup S_m)$ could be added to the collection). 
By Pigeonhole, at least $\frac{\alpha}{4m}n$ points must lie in a specific $S_i$, and since 
the $S_j$ are disjoint we have $m \leq \frac{4}{\alpha}$. Therefore, 
$|S \cap S_i| \geq \frac{\alpha^2}{16}n$ for some $i$. 
But by Proposition~\ref{prop:resilience-large} applied to $S_j$, 
the mean of $S \cap S_j$ must be within $\oo(\rho + B / \alpha^{1/t})$ of the 
mean $\hat{\mu}_j$ of $S_j$. Similarly, by Proposition~\ref{prop:resilience-large} 
applied to $S$, it is within $\oo(\rho + B/\alpha^{1/t})$ of $\mu$. 
Therefore, $\|\mu - \hat{\mu}_j\|_2 \leq \oo(\rho + B/\alpha^{1/t}) = \oo(\sqrt{RB/\alpha^{1/t}} + B / \alpha^{1/t})$, 
as was to be shown.
\end{proof}

\begin{algorithm}
\caption{Algorithm for re-clustering the $w_i$}
\label{alg:re-clustering}
\label{alg:re-cluster}
\setlength{\abovedisplayskip}{0pt}
\setlength{\belowdisplayskip}{0pt}
\setlength{\abovedisplayshortskip}{0pt}
\setlength{\belowdisplayshortskip}{0pt}
\begin{algorithmic}[1]
\STATE Input: $x_1,\ldots,x_n$, $\goodfrac$, $B$, and upper bound $r$ on $\|\mu\|_2$
\STATE Initialize $R \gets r$, $U \gets \{0\}$.
\STATE $r_{\mathrm{final}} \gets \Theta(B \goodfrac^{-1/m})$.
\WHILE{$R \geq r_{\mathrm{final}}$}
  \STATE $b \gets 1$
  \FOR{$\hat{\mu} \in U$}
    \STATE Let $w_{1:n}^{(b)}$, $c_{1:n}^{(b)}$ be the output of Algorithm~\ref{alg:clustering} 
           centered at $\hat{\mu}$ with parameters $B$, $\goodfrac$, $R$.
    \STATE $b \gets b+1$
  \ENDFOR
  \STATE Let $S_1, \ldots, S_m$ be the maximal covering derived from $w_{1:n}^{(1:b)}$, $c_{1:n}^{(1:b)}$ as 
         in Proposition~\ref{prop:covering}.
  \STATE $U \gets \{\hat{\mu}_j \}_{j=1}^m$, where $\hat{\mu}_j$ is the mean of $S_j$
  \STATE $R \gets C'(\sqrt{RB/\alpha^{1/t}} + B/\alpha^{1/t})$.
\ENDWHILE
\RETURN $U$
\end{algorithmic}
\end{algorithm}

This leads to Algorithm~\ref{alg:re-cluster}, which iteratively runs 
Algorithm~\ref{alg:cluster}, using Proposition~\ref{prop:covering} to find a set of 
at most $\frac{4}{\alpha}$ means such that one of them grows increasingly close to $\mu$
with each iteration. By repeatedly applying Proposition~\ref{prop:covering}, we can show 
that 
the output $U = \{\hat{\mu}_1, \ldots, \hat{\mu}_m\}$ of 
Algorithm~\ref{alg:re-cluster} satisfies 
$m \leq \frac{4}{\alpha}$ and $\|\hat{\mu}_j - \mu\|_2 \leq \oo(B/\alpha^{1/t})$, 
which yields Theorem~\ref{thm:robust-clustering-alpha}.

\subsection{Theorem~\ref{thm:clustering-separated}: Tighter bounds for well-separated clusters}
Suppose that rather than having a single good set $I$ with associated mean 
$\mu$, we have $k$ good sets $I_1, \ldots, I_k$, with corresponding 
means $\mu_1, \ldots, \mu_k$ and sizes $\alpha_1n, \ldots, \alpha_kn$, 
together with a fraction $\epsilon = 1-(\alpha_1 + \cdots + \alpha_k)$ 
of arbitrary outliers. In this case, assuming that the $\mu_j$ are 
\emph{well-separated}, we can get substantially stronger bounds by showing that every 
$S_j$ obtained in Algorithm~\ref{alg:re-cluster} is almost entirely a subset of one of the 
$I_{j'}$.
This is what yields Theorem~\ref{thm:clustering-separated}.

\begin{proof}[Proof of Theorem~\ref{thm:clustering-separated}]
%
We first show that in the final iteration of Algorithm~\ref{alg:re-cluster}, 
the $S_j$ from Proposition~\ref{prop:covering} must all have their 
mean $\hat{\mu}_j$ be close to a true mean $\mu_{j'}$. 
Indeed, a given $S_j$ has at most $\epsilon n \leq \frac{1}{8}\alpha n$ outlier points, 
and hence at least $\frac{1}{8}\alpha n$ points from $I_1 \cup \cdots \cup I_k$. 
By Pigeonhole it therefore has at least 
$\frac{1}{8}\alpha \cdot \alpha_{j'} n$ points from some $I_{j'}$. By 
Proposition~\ref{prop:resilience-large}, the mean of these 
points is within distance $\oo(B / \alpha^{1/2t})$ of $\mu_{j'}$, and within 
distance $\oo(B / \alpha^{1/t})$ of $\hat{\mu}_j$, and so 
$\|\mu_{j'} - \hat{\mu}_j\|_2 \leq \oo(B / \alpha^{1/t})$.

Moreover, each cluster $S_j$ is mostly ``pure'' in the sense that most points come from 
a single $I_{j'}$. Indeed, suppose for the sake of contradiction 
that more than $\frac{1}{4} \delta \alpha \alpha_{j''}$ 
points came from any $I_{j''} \neq I_{j'}$, with $\delta = \Theta(\Csep^{-2t})$. 
Then the mean of $S_j \cap I_{j''}$ would be 
within $\oo(B / (\alpha^2 \delta)^{1/2t}) = \frac{1}{3}\Csep B / \alpha^{1/t}$ of 
$\mu_{j''}$, while we already had that the mean of $S_j \cap I_{j'}$ was within 
$\oo(B / \alpha^{1/2t}) \leq \frac{1}{3}C_0 B / \alpha^{1/t}$ of $\mu_{j'}$ (for an 
appropriate constant $C_0$).
But $\|\mu_{j'} - \mu_{j''}\|_2 \geq \Csep B / \alpha^{1/t}$ and $S_j$ has radius 
$\oo(B / \alpha^{1/t}) < \frac{1}{3}C_0 B / \alpha^{1/t}$, which yields a contradiction whenever 
$\Csep \geq C_0$. Summing over ${j''}$, we have 
that at most $\frac{1}{4}\delta \alpha n \leq \delta |S_j|$ points 
come from any $I_{j''}$ with $j'' \neq j'$ (here we use $|S_j| \geq \frac{\alpha}{4}n$).

Now, because all of the sets $S_j$ have mean close to some $\mu_{j'}$, and the 
$\mu_{j'}$ are all far apart, we can consolidate the sets $S_j$ into $k$ new sets 
$\tilde{S}_1, \ldots, \tilde{S}_k$ (by e.g. merging together all $S_{j}$ whose 
means are within distance $(\Csep/4) \cdot B / \alpha^{1/t}$).
Each $\tilde{S}_j$ will have at most $\delta|\tilde{S}_j|$ points from $\cup_{j' \neq j} I_{j'}$, 
and will have at most $\epsilon n = (\epsilon / \alpha_j) |\tilde{S}_j|$ 
outliers. We then invoke Theorem~\ref{thm:robust-clustering-eps} 
with a $\epsilon/\alpha_j + \delta \leq 
\epsilon/\alpha + \oo(\Csep^{-2t})$ fraction of outliers;
as long as $\epsilon \leq \frac{\alpha}{8}$, the fraction of outliers will be below 
$\frac{1}{4}$ for sufficiently large $\Csep$, and hence we obtain the claimed error bound.
\end{proof}

\addreferencesection
\bibliographystyle{amsalpha}{}
\bibliography{bib/mathreview,bib/dblp,bib/scholar,bib/custom,refdb/all}

\appendix

\section{Proof of Lemma~\ref{lem:coalesce}}
\label{sec:coalesce-proof}

It suffices to prove the following four inequalities:
\begin{align}
\label{eq:majorize-odd}
u \langle M_a \otimes M_b, v^{\otimes (a+b)} \rangle &\preceq_{a+b} \langle M_{a-1} \otimes M_{b+1}, v^{\otimes (a+b)} \rangle &&\text{ if } a, b \text{ are odd, } a \leq b, \text{  and } u \in \{\pm 1\}. \\
\label{eq:majorize-even}
u \langle M_a \otimes M_b, v^{\otimes (a+b)} \rangle &\preceq_{a+b} \langle M_{a-2} \otimes M_{b+2}, v^{\otimes (a+b)} \rangle &&\text{ if } a, b \text{ are even, } a \leq b, \text{  and } u \in \{\pm 1\}. \\
\label{eq:combine}
u \langle M_a \otimes M_b, v^{\otimes (a+b)} \rangle &\preceq_{a+b} \langle M_{a+b}, v^{\otimes (a+b)} \rangle &&\text{ if } a+b \text{ is even and } u \in \{\pm 1\}. \\
\langle A \otimes B, v^{\otimes (a+b)} \rangle &\preceq_{a+b} \langle A' \otimes B', v^{\otimes (a+b)} \rangle &&\text{ if } 0 \preceq_{a} \langle A, v^{\otimes a} \rangle \preceq_{a} \langle A', v^{\otimes a} \rangle \text{ and } \notag \\
\label{eq:product}
 & && \ \ \ \ 0 \preceq_{b} \langle B, v^{\otimes b} \rangle \preceq_{b} \langle B', v^{\otimes b} \rangle.
\end{align}
Given (\ref{eq:majorize-odd}-\ref{eq:product}), we can use \eqref{eq:combine} to reduce the number of terms $M_i$ in the tensor product in the expression in the claim until there are at most $2$ left (such that each term $M_i$ has $i \leq 2t-2$). We then use \eqref{eq:majorize-odd} and \eqref{eq:majorize-even} to iteratively make the indices in the final 
two terms smaller and larger respectively, until one is $2$ and the other is $2t-2$. This will give us the claim. 

Thus, it suffices to establish the inequalities above.

Proof of \eqref{eq:product}: follows from the factorization $A' \otimes B' - A \otimes B = A' \otimes (B' - B) + (A' - A) \otimes B$.

Proof of \eqref{eq:majorize-odd}: first we need to show that 
$x^r -u x^{r-1}y - u xy^{r-1} + y^r \ssos 0$ (as a polynomial in $x$ and $y$) 
for any even $r$ and $u \in \{\pm 1\}$. We can WLOG take $u = 1$, as if $u = -1$ 
we just apply the same argument to the variables $(x,-y)$ instead of $(x,y)$. For $u = 1$ we have
\begin{align}
x^r - x^{r-1}y - xy^{r-1} + y^r 
 &= (x-y)(x^{r-1} - y^{r-1}) \\
 &= (x-y)^2(x^{r-2} + x^{r-1}y + x^{r-2}y^2 + \cdots + y^{r-2}) \\
 &= \frac{1}{2}(x-y)^2((x^{r/2-1})^2 + \sum_{i=1}^{r/2-2} (x^{r/2-i}y^{i-1} + x^{r/2-i-1}y^i)^2 + (y^{r/2-1})^2) \\
 &\ssos 0,
\end{align}
as claimed. Now \eqref{eq:majorize-odd} follows as it is equivalent to 
$x^{a-1}y^{b+1} - u x^ay^b - u x^by^a + x^{b+1}y^{a-1}$, which we get by setting $r = b-a+2$ and multiplying by $x^{a-1}y^{a-1}$.

Proof of \eqref{eq:majorize-even}: note that the $u = -1$ case is trivial as then the 
left-hand-side is $\psos 0$ while the right-hand-side is $\ssos 0$. So we can assume 
$u = 1$. Similarly to the previous proof, it suffices to show that 
$x^r - x^{r-2}y^2 - x^2y^{r-2} + y^r \ssos 0$. This is actually easier than before:
\begin{align}
x^r - x^{r-2}y^2 - x^2y^{r-2} + y^r
 &= (x^2-y^2)(x^{r-2} - y^{r-2}) \\
 &= (x^2 - y^2)^2(x^{r-4} + x^{r-6}y^2 + \cdots + y^{r-2}) \ssos 0.
\end{align}
Again, setting $r = b-a+4$ and multiplying by $x^{a-2}y^{a-2}$ yields the desired result.

Proof of \eqref{eq:combine}: follows from repeated application of \eqref{eq:majorize-odd} and \eqref{eq:majorize-even}. $\qed$

\section{Proof of Theorem~\ref{thm:poincare-composition}}
\label{sec:poincare-composition-proof}

First, we establish the product property. Suppose that $X_1, X_2$ both satisfy 
\eqref{eq:poincare-s3} with constants $\sigma_1, \sigma_2$ which are both at most $\sigma$. 
Supposing that $X_1 \in \bR^{d_1}$ and $X_2 \in \bR^{d_2}$, and that 
$f : \bR^{d_1} \times \bR^{d_2} \to \bR$, we have
\begin{align}
\Var_{x_1,x_2}[f(x_1,x_2)] 
 &\stackrel{(i)}{=} \Var_{x_1}[\bE_{x_2}[f(x_1,x_2) \mid x_1]] + \bE_{x_1}[\Var_{x_2}[f(x_1,x_2) \mid x_1]] \\
 &\stackrel{(ii)}{\leq} \sigma^2 \Big(\bE_{x_1}[\|\nabla_{x_1} \bE_{x_2}[f(x_1,x_2)]\|_2^2] + \bE_{x_1}[\bE_{x_2}[\|\nabla_{x_2} f(x_1,x_2)\|_2^2]]\Big) \\
 &= \sigma^2 \Big(\bE_{x_1}[\|\bE_{x_2}[\nabla_{x_1}f(x_1,x_2)]\|_2^2] + \bE_{x_1,x_2}[\|\nabla_{x_2} f(x_1,x_2)\|_2^2]]\Big) \\
 &\stackrel{(iii)}{\leq} \sigma^2 \Big(\bE_{x_1,x_2}[\|\nabla_{x_1}f(x_1,x_2)\|_2^2] + \bE_{x_1,x_2}[\|\nabla_{x_2} f(x_1,x_2)\|_2^2]]\Big) \\
 &= \sigma^2 \bE_{x_1,x_2}[\|\nabla f(x_1,x_2)\|_2^2].
\end{align}
Here (i) is the law of total variation, (ii) applies \eqref{eq:poincare-s3} for the 
individual random variables $X_1$ and $X_2$, and (iii) is Jensen's inequality for the 
Euclidean norm.

Next, we establish the Lipschitz composition property. Suppose that 
$X \in \bR^d$ satisfies \eqref{eq:poincare-s3} with constant $\sigma$, 
and that $\Phi : \bR^d \to \bR^{d'}$ satisfies $\|\Phi(x) - \Phi(x')\|_2 \leq L\|x-x'\|_2$ 
for all $x, x' \in \bR^d$. Then for $f : \bR^{d'} \to \bR$, we have
\begin{align}
\Var_x[f(\Phi(x))] 
 &\stackrel{(i)}{\leq} \sigma^2 \bE_{x}[\|\nabla (f \circ \Phi)(x)\|_2^2] \\
 &\stackrel{(ii)}{=} \sigma^2 \bE_{x}[\|\frac{\partial \Phi}{\partial x} \nabla f(\Phi(x))\|_2^2] \\
 &\leq \sigma^2 \bE_{x}[\|\frac{\partial \Phi}{\partial x}\|_{\op}^2 \|\nabla f(\Phi(x))\|_2^2] \\
 &\stackrel{(iii)}{\leq} \sigma^2 L^2 \bE_{x}[\|\nabla f(\Phi(x))\|_2^2].
\end{align}
Here (i) is \eqref{eq:poincare-s3} applied to the function $f \circ \Phi$, 
(ii) is the chain rule, and (iii) is the assumed Lipschitz bound on $\Phi$.

Finally, we establish the linear combination property, which follows from the two 
properties above. Suppose that $X_1, X_2$ satisfy \eqref{eq:poincare-s3} with 
constants $\sigma_1, \sigma_2$. Then by the product property, 
the random variable $(X_1/\sigma_1, X_2/\sigma_2)$ satisfies \eqref{eq:poincare-s3} 
with constant $1$. But now the map 
$(Y_1, Y_2) \mapsto a\sigma_1 Y_1 + b\sigma_2Y_2$ is Lipschitz with constant 
$L = \sqrt{a^2\sigma_1^2 + b^2\sigma_2^2}$, so in particular 
$aX_1 + bX_2$ satisfies \eqref{eq:poincare-s3} with constant 
$\sqrt{a^2\sigma_1 + b^2\sigma_2^2}$. This completes the proof.

\section{Proof of Lemma~\ref{lem:basic-relaxation}}
\label{sec:basic-relaxation-proof}

Note that the optimization problem 
\begin{equation}
\label{eq:minimax}
\min_{w_1, \ldots, w_n} \max_{\xi, \zeta_1, \ldots, \zeta_n} \sum_{i=1}^n c_i \pE_{\xi(v)}[\iprod{ x_i - w_i, v }^{2t}] + \lambda \pE_{\zeta_i(z_i)}[\iprod{ w_i, z_i }^{2t}]
\end{equation}
is a convex-concave saddle point problem. In particular, convexity in $w$
follows because the convexity (in $v$) of the function 
$\iprod{x_i - w_i, v}^{2t}$ has a sum-of-squares proof, and similarly 
for $\iprod{w_i, z_i}^{2t}$. On the other hand, concavity in 
$\xi, \zeta$ holds because the objective is in fact linear in $z_i$. 

In particular, by the minimax theorem, \eqref{eq:minimax} remains unchanged 
if we swap the order of min and max. However, we also need to check that 
both directions of the saddle point problem can be efficiently computed.

To do this, we first relax $w_1, \ldots, w_n$ to themselves be pseudodistributions 
$\theta_i(w_i)$. Then the entire objective for \eqref{eq:minimax} becomes bilinear 
in $\theta$ and in $\xi, \zeta$. Moreover, 
since sum-of-squares programs can be represented as semidefinite programs, 
the relaxed version of \eqref{eq:minimax} is in fact a bilinear min-max problem 
with semidefinite constraints on the two sets of variables. It is a standard 
fact in optimization that such programs can be solved efficiently to arbitrary precision.

Now, to finish, we note that the relaxation of $w_i$ was not in fact a relaxation at all! 
Because \eqref{eq:minimax} was already convex in $w_i$, there is always an optimal 
$\theta_i$ that places all of its mass on a single point. In particular, given optimal 
values $\theta_1^{\star}, \ldots, \theta_n^{\star}$, we can simply 
take $w_i^{\star} = \pE_{\theta_i^{\star}(w)}[w]$, which must have at least as good of 
an objective value for \eqref{eq:minimax} (because of the sum-of-squares proof of convexity).

\section{Proof of Propositions~\ref{prop:resilience-small} and \ref{prop:resilience-large}}
\label{sec:resilience-proof}

The proofs of Proposition~\ref{prop:resilience-small} and \ref{prop:resilience-large} 
are almost identical, so we prove them together.

We start with Proposition~\ref{prop:resilience-large}. We write
\begin{align}
\Big\|\frac{1}{|S'|} \sum_{i \in S'} (x_i - \hat{\mu})\Big\|_2 
 &= \Big\|\frac{1}{|S'|} \sum_{i \in S'} (x_i - w_i) + (w_i - \hat{\mu})\Big\|_2 \\
 &\leq \rad + \sup_{\|v\|_2 \leq 1} \frac{1}{|S'|} \sum_{i \in S'} \langle x_i - w_i, v \rangle \\
 &\leq \rad + \sup_{\|v\|_2 \leq 1} \Big(\frac{1}{|S'|} \sum_{i in S'} \langle x_i - w_i, v \rangle^{2t}\Big)^{1/2t} \\
 &\leq \rad + \sup_{\|v\|_2 \leq 1} \Big(\frac{4}{|S'|} \sum_{i=1}^n c_i \langle x_i - w_i, v \rangle^{2t}\Big)^{1/2t} \\
 &= \rad + \Big(\frac{4}{|S'|} \oo(nB^{2t})\Big)^{1/2t} \\
 &= \rad + \oo(B / \beta^{1/2t})
\end{align}
as was to be shown.

We now consider Proposition~\ref{prop:resilience-small}. We have
\begin{align}
\Big\|\frac{1}{|S'|} \sum_{i \in S'} (x_i - \hat{\mu})\Big\|_2 
 &= \Big\|\frac{1}{|S'|} \sum_{i \in S\backslash S'} (x_i - \hat{\mu})\Big\|_2 \\
 &= \frac{\epsilon_1}{1-\epsilon_1} \Big\|\frac{1}{|S \backslash S'|} \sum_{i \in S \backslash S'} (x_i - \hat{\mu})\Big\|_2 \\
 &\stackrel{(i)}{\leq} \frac{\epsilon_1}{1-\epsilon_1} (\rad + \oo(B / (\epsilon_1\beta)^{1/2t})) \\
 &\leq 2\epsilon_1 \rad + \oo(B \epsilon_1^{1-1/2t} \beta^{-1/2t}),
\end{align} 
where (i) simply applies Proposition~\ref{prop:resilience-large}.

\end{document}